\documentclass[a4paper,10pt]{article}

\usepackage[T1]{fontenc}
\usepackage[utf8]{inputenc}
\usepackage{tgtermes}

\usepackage[
    a4paper,
    top=1.30in,
    bottom=1.30in,
    left=0.90in,
    right=0.90in
]{geometry}
\usepackage{microtype}
\usepackage{setspace}
\usepackage{amsmath,amssymb,amsthm,mathrsfs}

\theoremstyle{plain}
\newtheorem{theorem}{Theorem}
\newtheorem{corollary}{Corollary}
\newtheorem{assumption}{Assumption}
\newtheorem{condition}{Condition}
\newtheorem{proposition}{Proposition}
\newtheorem{lemma}{Lemma}
\newtheorem{remark}{Remark}
\newtheorem{definition}{Definition}
\usepackage{algorithm}
\usepackage{algorithmic}

\AtBeginDocument{%
    \setlength{\abovedisplayskip}{3pt}%
    \setlength{\belowdisplayskip}{3pt}%
}

\usepackage{titlesec}

\titleformat{\section}[block]
    {\normalfont\normalsize\bfseries}
    {\thesection.\ \ }
    {0em}
    {}

\titleformat{\subsection}[block]
    {\normalfont\normalsize\bfseries}
    {\thesubsection\ \ }
    {0em}
    {}

\usepackage{authblk}

\usepackage[round,authoryear]{natbib}
\setcitestyle{citesep={,}}
\usepackage{xcolor}
\usepackage[
    colorlinks=true,
    linkcolor=blue!50!black,
    citecolor=blue!50!black,
    urlcolor=blue!50!black
]{hyperref}

\newcommand{\R}{\mathbb{R}}

\newcommand{\Tl}{T_{\lambda}}
\newcommand{\Kl}{K_{\lambda}}

\newcommand{\dpp}{\nabla_{p}}

\DeclareMathOperator{\diag}{diag}
\DeclareMathOperator{\Tr}{Tr}

\makeatletter
\renewcommand\subsubsection{\@startsection{subsubsection}{3}{\z@}%
  {-3.25ex\@plus -1ex \@minus -.2ex}%
  {1.5ex \@plus .2ex}%
  {\normalfont\small\itshape\mdseries}}
\makeatother
\titleformat{\section}[block]
  {\normalfont\normalsize\bfseries}{\thesection.\ \ }{0em}{}
\titleformat{\subsection}[block]
  {\normalfont\normalsize\itshape}{\thesubsection.\ \ }{0em}{}
\titleformat{\subsubsubsection}[block]
  {\normalfont\small\itshape\mdseries}{\thesubsubsubsection.}{0em}{}
\title{\Large Muon meets Tamed Langevin: Momentum Preconditioning beyond Convex and gradient-Lipschitz Potentials}
\author[1]{Nikolaos Makras}
\author[1,2,3]{Sotirios Sabanis}
\affil[1]{School of Mathematics, University of Edinburgh, UK}
\affil[2]{National Technical University of Athens, Greece}
\affil[3]{Athena/Archimedes Research Centre, Greece}
\date{\small \today}
\begin{document}

\maketitle 
\begin{abstract}
We consider the problem of sampling from Gibbs distributions on matrix spaces whose potential energies are neither convex nor globally gradient-Lipschitz. We introduce a family of non-quadratic kinetic energies that lead to a new underdamped Langevin system with momentum preconditioning, in which the gradient of the kinetic energy acts as a smooth spectral taming of the momentum. We prove that, under these relaxed assumptions on the potential, the resulting dynamics leaves the target Gibbs measure invariant, and we establish exponential convergence to equilibrium in a weighted total variation distance. Finally, we show that the corresponding Euler–Maruyama discretization admits moment bounds that are uniform in time, without any modification of the potential gradient, which ensures the stability of the resulting sampling algorithm.

\end{abstract}
\section{Introduction}
\label{sec:introduction}

Muon has emerged as a prominent optimiser for large language model pretraining, with reported state-of-the-art results in training efficiency and substantial empirical gains over AdamW \citep{Jordan_Muon_2024,LiuEtAlMuon2025,ShahEtAl2025}. Its defining feature is the orthogonalisation of matrix-valued momentum, which exploits the matrix structure of neural-network parameters. For a differentiable objective $U:\R^{n\times m}\to\R$, the idealised Muon update with exponential moving average (EMA) momentum is
\begin{align}
    m_{k+1}&=\beta_1m_k+(1-\beta_1)\nabla U(X_k),\nonumber\\
    X_{k+1}&=X_k-\eta\operatorname{Orth}(m_{k+1}),
    \label{intro:muon}
\end{align}
where $\beta_1\in[0,1)$ and $\eta>0$. If $p=U\Sigma V^\top$ is a reduced singular value decomposition containing only the strictly positive singular values, then $\operatorname{Orth}(p)=UV^\top$, with $\operatorname{Orth}(0)=0$. Thus, idealised orthogonalisation preserves the singular directions
of the momentum and replaces its nonzero singular values by one.
In practice, an exact SVD is computationally expensive on GPUs,
so Muon instead uses a small number of Newton-Schulz iterations,
implemented through efficient matrix multiplications after
Frobenius normalisation \citep{Jordan_Muon_2024}.
With finitely many iterations, sufficiently small normalised
singular values remain below one, attenuating their associated
directions relative to exact orthogonalisation
\citep{MagakyanParriloOzdaglar2026}.
This implicit spectral filtering may also have a beneficial
regularising effect. In particular, \citet{PangEtAl2026} report
that approximate orthogonalisation can outperform its exact
SVD counterpart and interpret this behaviour as downweighting
directions dominated by noise.

Despite this empirical success and its ever growing adoption \citep{KimiK2_2025,Karpathy_nanochat_2025} in LLM pretraining, global convergence remains a delicate issue. Stochastic Muon with fixed mini-batches can fail to converge to a global minimiser even for smooth, strongly convex objectives \citep{DoDereichJentzen2026}. Nonconvergence examples also arise for nonsmooth convex Lipschitz objectives \citep{ParshakovaEtAl2026}. These results highlight limitations of direct optimisation guarantees and motivate studying Muon's spectral normalisation within a stochastic framework whose long-time behaviour can be characterised rigorously.

Our starting observation is that spectral normalisation is closely related to \emph{taming}, a technique used to stabilise explicit stochastic numerical schemes when coefficients grow superlinearly \citep{HutzenthalerJentzenKloeden2012,Sabanis2013}. For example, coordinate-wise taming in the tamed unadjusted Langevin algorithm rescales large gradient components while retaining an approximately linear response near zero \citep{BrosseDurmusMoulinesSabanis2019}. A smooth regularisation of Muon's orthogonalisation has the same qualitative effect on the singular values of the momentum. This suggests using singular-value taming to connect Muon with Langevin sampling and leads to the following question.

\begin{quote}
\textit{(Q) Can we design a stable Langevin-based algorithm with theoretical guarantees, whose zero-temperature limit recovers $\varepsilon$-regularised Muon?}
\end{quote}

We address this question by incorporating singular-value taming into the momentum-to-position map of a kinetic Langevin system. A suitable choice of momentum diffusion and divergence correction preserves both linear friction and the desired Gibbs position marginal. Under assumptions allowing nonconvex potentials and superlinear gradient growth, we establish exponential convergence of the diffusion and moment bounds for its explicit discretisation that are uniform in the number of iterations.

To describe the construction, fix an inverse temperature $\beta>0$ and the target measure
\begin{equation}
    \pi_\beta^x(dx)=Z_\beta^{-1}e^{-\beta U(x)}\,dx,
    \qquad
    Z_\beta=\int_{\R^{m\times n}}e^{-\beta U(x)}\,dx<\infty.
    \label{intro:gibbs}
\end{equation}
For a taming parameter $\delta>0$, introduce the generalized kinetic energy
\begin{equation}
    K_\delta(p)=\frac{1}{\delta}\operatorname{Tr}\!\left[
        (I_n+\delta pp^\top)^{1/2}-I_n
    \right].
    \label{intro:kinetic}
\end{equation}
Its gradient is the singular-value taming map
\begin{equation}
    T_\delta(p)=\nabla_pK_\delta(p)
    =(I_n+\delta pp^\top)^{-1/2}p
    =U\operatorname{diag}\!\left(
        \frac{\sigma_i(p)}{\sqrt{1+\delta\sigma_i(p)^2}}
    \right)V^\top,
    \label{intro:taming}
\end{equation}
which alike the preconditioner of Muon, acts on the singular values only. The scalar map $s\mapsto s/\sqrt{1+\delta s^2}$ is approximately
linear near zero and saturates smoothly at $\delta^{-1/2}$.
This scalar taming structure also appears in Langevin
algorithms and SDE approximations \citep{NeufeldNgZhang2025,SabanisEuler2016}. Moreover, it coincides with the relativistic momentum-to-velocity relation for unit rest mass and light speed $\delta^{-1/2}$, motivating the terminology \emph{relativistic kinetic energy} \citep{Debbasch1997}. The same smooth spectral potential, up to scaling, has also appeared in the deterministic Hamiltonian formulation of regularised Muon in \citet{MustafiMukherjeeSriperumbudur2026}. Here, we use it to construct and analyse stochastic dynamics with the prescribed position marginal \eqref{intro:gibbs}.

The derivation of the taming mechanism through the choice of Kinetic energy is key to our construction. Arbitrary replacing the force in overdamped Langevin dynamics by a tamed approximation, while keeping the diffusion coefficient unchanged, generally alters the invariant distribution and leads to overtamed algorithms. We instead tame the momentum entering the position equation and retain the original gradient force in the momentum equation. To preserve linear friction, define
\begin{equation}
    M_\delta(p)=(I_n+\delta pp^\top)^{1/2},
    \qquad
    [\Gamma_\delta(p)]_{ij}
    =\sum_{a=1}^{n}
        \frac{\partial[M_\delta(p)]_{ia}}{\partial p_{aj}},
    \label{intro:mobility}
\end{equation}
for $1\leq i\leq n$ and $1\leq j\leq m$. The identity $M_\delta(p)T_\delta(p)=p$, together with the SDE to invariant measure recipe of \citet{MaChenFox2015}, yields
\begin{align}
    dX_t&=T_\delta(P_t)\,dt,\nonumber\\
    dP_t&=\left[-\nabla U(X_t)-\gamma P_t
        +\frac{\gamma}{\beta}\Gamma_\delta(P_t)\right]dt
        +\sqrt{\frac{2\gamma}{\beta}}\,
            M_\delta(P_t)^{1/2}\,dW_t,
    \label{intro:sde}
\end{align}
where $\gamma>0$ is the friction parameter and $W_t$ is an $n\times m$ matrix of independent standard Brownian motions. The state-dependent noise and its divergence correction are chosen jointly with the linear damping. We show that the latter diffusion admits the Gibbs invariant law
associated with the separable Hamiltonian
$H_\delta(x,p)=U(x)+K_\delta(p)$:
\begin{equation}
    \mu_{\beta,\delta}(dx,dp)
    =Z_{\beta,\delta}^{-1}
        e^{-\beta[U(x)+K_\delta(p)]}\,dx\,dp
    =\pi_\beta^x(dx)\,\pi_{\beta,\delta}^p(dp),
    \label{intro:joint}
\end{equation}
where $\pi_{\beta,\delta}^p(dp)$ is proportional to $e^{-\beta K_\delta(p)}dp$. Singular-value taming therefore changes the equilibrium momentum distribution while preserving the intended position marginal exactly.

The Euler-Maruyama discretisation of \eqref{intro:sde} with stepsize $\lambda>0$ gives the explicit scheme
\begin{align}
    X_{k+1}&=X_k+\lambda T_\delta(P_k),\nonumber\\
    P_{k+1}&=(1-\gamma\lambda)P_k-\lambda\nabla U(X_k)
        +\frac{\lambda\gamma}{\beta}\Gamma_\delta(P_k)
        +\sqrt{\frac{2\lambda\gamma}{\beta}}\,
            M_\delta(P_k)^{1/2}\xi_{k+1},
    \label{intro:algorithm}
\end{align}
where $(\xi_k)_{k\geq1}$ are i.i.d. standard Gaussian matrices in $\R^{n\times m}$. The resulting scheme is also known as the unadjusted Langevin
algorithm (ULA). For the classical overdamped setting, see
\citet{MeynTweedie1993,DurmusMoulines2017} and references therein. For fixed $\delta$, $\gamma$ and $\lambda$, the stochastic increment and divergence correction vanish as $\beta\to\infty$. Setting
 $m_k=-\gamma P_k,
    \ \beta_1=1-\gamma\lambda,
    \ \eta=\lambda/\gamma, \text{with}
    \ 0<\gamma\lambda<1$,
the limiting iterative scheme can be rewritten as
\begin{align}
    m_{k+1}&=\beta_1m_k+(1-\beta_1)\nabla U(X_k),\nonumber\\
    X_{k+1}&=X_k-\eta T_{\delta/\gamma^2}(m_k).
    \label{intro:ema}
\end{align}
Indeed, defining the $\varepsilon$-regularised orthogonalisation map by
\begin{equation}
    \operatorname{Orth}_\varepsilon(p)
    =U\operatorname{diag}\!\left(
        \frac{\sigma_i(p)}{\sqrt{\varepsilon^2+\sigma_i(p)^2}}
    \right)V^\top,
    \qquad
    T_{\delta/\gamma^2}(p)
    =\frac{\gamma}{\sqrt\delta}\,
        \operatorname{Orth}_{\gamma/\sqrt\delta}(p),
    \label{intro:regularisation}
\end{equation}
identifies \eqref{intro:ema} as regularised Muon with regularisation $\varepsilon=\gamma/\sqrt\delta$ and effective learning rate $\lambda/\sqrt\delta$. We note that, the position update uses the previous momentum $m_k$, as dictated by explicit Euler discretisation.

The large-$\beta$ regime used above to recover the regularised
Muon structure also provides a natural connection between sampling
and optimisation. In the context of simulated annealing, since $\pi_\beta^x(dx)\propto e^{-\beta U(x)}dx$,
increasing $\beta$ significantly favours samples with lower objective values.
Under suitable assumptions, this measure concentrates around the
global minimisers of $U$ as $\beta\to\infty$ \citep{Hwang1980}. Quantitative forms of this connection yield non-asymptotic
optimisation guarantees.

\subsection{Contributions}
\label{sec:intro_contributions}

We study $C^\infty$, semiconvex potentials with polynomially Lipschitz gradients, an $L_0/L_1$ Hessian growth bound and convexity at infinity. This class includes objective functions of LLMs with smooth activation
functions, subject to the stated growth and curvature assumptions,
see \citet{LytrasMakrasSabanis2026} for further discussion. Our main contributions are as follows.
\begin{itemize}
    \item We construct the kinetic Langevin system \eqref{intro:sde} and its explicit discretisation \eqref{intro:algorithm}. The discrete scheme recovers the regularised Muon recursion \eqref{intro:ema} in the zero-temperature limit.

    \item We identify the Gibbs invariant law \eqref{intro:joint}, whose position marginal is exactly the prescribed target \eqref{intro:gibbs}, independently of the kinetic regularisation parameter.

    \item We establish exponential convergence of the diffusion in weighted total variation. The proof combines Lyapunov estimates with hypoellipticity, controllability and a minorisation argument.

    \item We prove that the algorithm is stable, in the sense that it admits moment bounds that are uniform-in-iterations, despite the scheme being explicit and the gradient may grow superlinearly. 
\end{itemize}
Our two main analytical results can be stated informally as follows.

\begin{theorem}[Exponential convergence, informal]
\label{intro:theorem1}
Under our standing assumptions and for admissible parameter choices, the diffusion \eqref{intro:sde} admits the unique invariant probability measure $\mu_{\beta,\delta}$ in \eqref{intro:joint}. There exist a Lyapunov weight $V\geq1$ and constants $C,c>0$ such that, for every initial law $\mu_0$ satisfying $\mu_0(V)<\infty$,
\[
    \bigl\|\mathcal{L}(X_t,P_t)-\mu_{\beta,\delta}\bigr\|_{\mathrm{TV},V}
    \leq Ce^{-ct}\mu_0(V),
    \qquad t\geq0,
\]
where $\|\cdot\|_{\mathrm{TV},V}$ denotes weighted total variation distance.
\end{theorem}

\begin{theorem}[Uniform moment bounds, informal]
\label{intro:theorem2}
Under our standing assumptions, for admissible parameter choices and a sufficiently small fixed stepsize $\lambda>0$, the explicit scheme \eqref{intro:algorithm} has polynomial moments of every order uniformly over all iterations, provided its initial state has a finite expectation of the exponential Lyapunov weight used in the discrete analysis. In particular, for every deterministic initial state $(X_0,P_0)=(x_0,p_0)$ and every $q\geq1$, there exists $C_q<\infty$, independent of $k$, such that
\[
    \sup_{k\geq0}\mathbb{E}\left[
        \|X_k\|_{\mathrm{F}}^q+\|P_k\|_{\mathrm{F}}^q
    \right]\leq C_q.
\]
\end{theorem}

\subsection{Related literature}
\label{sec:intro_related}

The closest comparisons concern Muon and its spectral variants. \citet{PethickEtAl2025} place Muon related methods in a framework of norm-constrained linear minimisation oracles and establish stationarity guarantees under smoothness and stochastic-gradient assumptions. \citet{QiEtAl2026} study spectral transformations and their empirical consequences. Smooth spectral saturation also appears in SoftMuon \citep{FeoktistovEtAl2026}, while Musec uses spectral clipping and provides stationarity guarantees for weakly convex, nonsmooth objectives with globally Lipschitz sample losses \citep{LiuWangLuo2026}. We highlight that these results mainly study convergence to equilibrium, which includes saddle points. Under strong convexity and Lipschitz-gradient assumptions, \citet{DoDereichJentzen2026} additionally derive error bounds relative to a global minimiser, including practical orthogonalisation polynomials, alongside the fixed-mini-batch nonconvergence result discussed above.

Particularly close to our construction, \citet{MustafiMukherjeeSriperumbudur2026} derive deterministic Hamiltonian probability dynamics from regularised Muon, using the same smooth spectral kinetic energy up to scaling. They obtain exponential objective-gap bounds under bounded gradient and curvature conditions. To the best of our knowledge, no existing work identifies Muon with the
zero-temperature limit of a sampling algorithm for a prescribed Gibbs measure, and our guarantees require
neither convexity nor globally Lipschitz gradients.

In stochastic numerical analysis, superlinear coefficient growth can make explicit Euler approximations unstable, with moments diverging as the stepsize tends to zero on a fixed finite time interval \citep{HutzenthalerJentzenKloeden2011}. Tamed Euler methods control this growth while retaining explicit updates \citep{HutzenthalerJentzenKloeden2012,Sabanis2013}. For sampling, \citet{BrosseDurmusMoulinesSabanis2019} introduce the tamed unadjusted Langevin algorithm and establish non-asymptotic Wasserstein-$2$ bounds. Related extensions address superlinear gradients \citep{LimSabanis2024} and nonsmooth, semiconvex potentials \citep{LytrasMakrasSabanis2026}. In the kinetic setting, the tKLMC1 scheme of \citet{JohnstonLytrasSabanis2024} tames the gradient force in the momentum update. Our construction instead places singularvalue taming in the momentum-to-position map, with momentum noise and damping chosen to preserve the Gibbs target at the level of the diffusion.

General kinetic energies also have been studied for underdamped Langevin systems. \citet{StoltzTrstanova2018} study Langevin dynamics with nonquadratic kinetic energies and prove exponential convergence in $L^2$. \citet{BernardFathiLevittStoltz2022} obtain quantitative $L^2$ resolvent estimates through a Schur-complement approach. For relativistic dynamics with singular particle interactions, \citet{DuongNguyen2024} establish polynomial convergence in total variation. These results provide relevant analytical context, although their dynamics and assumptions differ from those considered here.
\section*{Notation}
Let $A,B\in\mathbb{R}^{n\times m}$, we define their Frobenius inner product as $\langle A,B\rangle:=\Tr(A^\top B)=\sum_{i=1}^n\sum_{j=1}^mA_{ij}B_{ij}$. The induced Frobenius norm is $|A|:=\langle A,A\rangle^{1/2}$, which reduces to the Euclidean norm for vectors. Moreover, we define the spectral norm $|A|_2:=\sigma_1(A)=\sup_{|v|=1}|Av|$ and the nuclear norm $|A|_*:=\sum_{i=1}^d\sigma_i(A)$. We denote by $e_1,e_2,\ldots$ the canonical basis vectors, by $\mathbb{I}^{kj}:=e_ke_j^\top\in\R^{n\times m}$ the matrix units, i.e.\ $[\mathbb{I}^{kj}]_{ab}=\mathbb{I}\{a=k,\,b=j\}$. The latter is distinct from the indicator function of a set $A$ which we write as $\mathbb{I}_A$. Additionally, $I_k$ is the $k\times k$ identity matrix, where the subscript is dropped when the dimension is clear, and $\diag(v)$ is the, possibly rectangular, diagonal matrix with diagonal $v$. Henceforth, by $\odot$ we denote the Hadamard product and by $\otimes$ the Kronecker product. 

Let $f:\R^{n\times m}\to\R$ be the scalar function of a matrix-valued variable $x=[x_{ij}]_{i\leq n,\,j\leq m}\in\R^{n\times m}$, and write $\partial_{x_{ij}}=\partial/\partial x_{ij}$. The gradient and the Laplacian are understood in the following way
\[
    \nabla f=\left[\dfrac{\partial f}{\partial x_{ij}}\right]_{i\leq n,\,j\leq m}\in\R^{n\times m},
    \qquad
    \Delta f=\sum_{i\leq n}\sum_{j\leq m}\dfrac{\partial^2 f}{\partial x_{ij}^2}\in\R,
\]
and we write $\nabla_xf$, $\Delta_xf$ when the variable has to be indicated. One can vectorize $x$ into $\hat{x}:=\operatorname{vec}(x)\in\R^{nm}$ such that $\hat{x}_r=x_{i_rj_r}$, where $i_r=[(r-1)\bmod n]+1$ and $j_r=\lfloor(r-1)/n\rfloor+1$. Therefore the gradient can also be written as $\nabla_{\hat{x}}f:=\operatorname{vec}(\nabla_xf)=(\partial f/\partial\hat{x}_r)_{r=1}^{nm}\in\R^{nm}$. Moreover, let $F:\mathbb{R}^{n\times m}\to \mathbb{R}^{n\times m}$, its Jacobian is denoted as $\mathcal{J}_F$. Explicitly, $\mathcal{J}_F(x)$ is the linear map $(\mathcal{J}_F(x)[h])_{ij}=\sum_{k,l}\partial_{x_{kl}}F_{ij}(x)h_{kl}$, with adjoint $(\mathcal{J}_F(x)^\top[h])_{kl}=\sum_{i,j}\partial_{x_{kl}}F_{ij}(x)h_{ij}$, so that $\nabla\langle F(\cdot),h\rangle=\mathcal{J}_F^\top h$. With the customary abuse of notation, the Hessian $$\nabla^2f:=\mathcal{J}_{\nabla f}=[\partial^2 f/\partial x_{ij}\partial x_{kl}]$$ carries four indices and is viewed as a self-adjoint operator on $(\R^{n\times m},\langle\cdot,\cdot\rangle)$, acting entrywise by $(\nabla^2f(x)[h])_{ij}=\sum_{k,l}\partial_{x_{ij}}\partial_{x_{kl}}f(x)\,h_{kl}$. Moreover, under vectorization it is identified as the the symmetric $nm\times nm$ matrix of second partials.

Let $\mathscr{B}(E)$ denote the Borel $\sigma$-field of a metric space $E$, and $\mathcal{P}(E)$ the set of probability measures on $(E,\mathscr{B}(E))$. Moreover, let $\mathrm{L}^1(\mu)$ be the set of $\mu$-integrable functions for $\mu\in\mathcal{P}(E)$. Further, $\mu(f)=\int_Ef(x)\,d\mu(x)$ for any $f\in\mathrm{L}^1(\mu)$. Given a Markov kernel $R$ on $E$, for all $x\in E$ and $f$ integrable under $R(x,\cdot)$, denote by $Rf(x)=\int_Ef(y)R(x,dy)$. Let $V:E\to[1,\infty)$ be a measurable function. The $V$-total variation distance between $\mu,\nu\in\mathcal{P}(E)$ is defined as $\|\mu-\nu\|_V=\sup_{|f|\leq V}|\mu(f)-\nu(f)|$. If $V=1$, then $\|\cdot\|_V$ is the total variation denoted by $\|\cdot\|_{\mathrm{TV}}$, so that $\|\mu-\nu\|_{\mathrm{TV}}=2\sup_{A\in\mathscr{B}(E)}|\mu(A)-\nu(A)|$.
\section{Kinetic energy and derivation of the SDE}

Consider a parameter $\delta \in (0, 1]$, an inverse temperature $\beta > 0$, and a friction coefficient $\gamma > 0$. For state variables $x, p \in \mathbb{R}^{n \times m}$ and $d = \min(n, m)$, we define the kinetic energy as a singular value generalization of the pseudo Huber potential:
\begin{align}
    K_{\delta}(p)
    &= \frac{1}{\delta} \operatorname{Tr} \left[ (I_n + \delta p p^{\top})^{1/2} - I_n \right]
    = \sum_{i=1}^{d} k_{\delta}(\sigma_i(p)),
    \qquad
    k_{\delta}(s) = \frac{\sqrt{1 + \delta s^2} - 1}{\delta}.
    \label{kinEnergy}
\end{align}
Let $p = \mathsf{U}\Sigma\mathsf{V}^{\top}$ denote a singular value decomposition, where $\Sigma = \operatorname{diag}(\sigma_1(p), \ldots, \sigma_d(p))$. The gradient of the kinetic energy is given by
\begin{align}
    T_{\delta}(p) := \nabla_p K_{\delta}(p)
    &= (I_n + \delta p p^{\top})^{-1/2} p
    = p (I_m + \delta p^{\top} p)^{-1/2}
    = \mathsf{U}\tau_{\delta}(\Sigma)\mathsf{V}^{\top},
    \nonumber\\
    \text{ where }\tau_{\delta}(s) &= \frac{s}{\sqrt{1 + \delta s^2}}.
    \label{taming}
\end{align}
Consequently, the preconditioner effectively acts only on the singular values. Moreover, it is inherited to the scheme exclusively via the choice of kinetic energy.

\begin{remark}
The scalar map in \eqref{taming} coincides with the relativistic momentum to velocity relation
\begin{align}
    v(s) = \frac{s}{\sqrt{m^2 + s^2 / c^2}},
    \label{Lorentz}
\end{align}
for a unit rest mass $m = 1$ and a speed limit $c = \delta^{-1/2}$. Taking the limit $c \to \infty$ while holding $m = 1$ recovers the classical velocity $v(s) = s$. Conversely, the massless limit $m \to 0$ with $c = 1$ maps each positive singular value to one. This recovers Muon's orthogonalization, while zero singular values remain zero.
\end{remark}

We define the Hamiltonian as $H(x,p) = U(x) + K_{\delta}(p)$. Our goal is to formulate an underdamped Langevin system featuring a tamed position velocity, linear momentum friction, and a Gibbs density proportional to $e^{-\beta H(x,p)}$. Following the framework established by \citet{MaChenFox2015}, we denote $N = nm$, $\hat{x} = \operatorname{vec}(x)$, $\hat{p} = \operatorname{vec}(p)$, and the combined state vector $z = (\hat{x}, \hat{p}) \in \mathbb{R}^{2N}$. Given the stochastic differential equation
\begin{align}
    dZ_t = b(Z_t) \, dt + \sqrt{2D(Z_t)} \, d\widehat{W}_t,
    \label{mal_SDE}
\end{align}
where $\widehat{W}_t$ is a standard Brownian motion in $\mathbb{R}^{2N}$, the recipe for achieving the target density $e^{-\beta H}$ requires
\begin{align}
    b(z) &= -(D(z) + Q(z)) \nabla_z (\beta H(z)) + \Gamma(z),
    \nonumber\\
    \Gamma_r(z) &= \sum_{s=1}^{2N}
    \frac{\partial [D(z) + Q(z)]_{rs}}{\partial z_s},
    \qquad r = 1, \ldots, 2N,
    \label{mal_recipe}
\end{align}
where $D$ is a symmetric positive semidefinite matrix and $Q$ is a skew symmetric matrix.

Let $\bar{D}(p) \in \mathbb{R}^{n \times n}$  be the smooth, symmetric, and positive definite matrix corresponding to the momentum's diffusion coefficient.  Then we define $D,Q$ as the block matrices
\begin{align*}
    D(z) &= \frac{1}{\beta}
    \begin{pmatrix}
        0 & 0 \\
        0 & I_m \otimes \bar{D}(p)
    \end{pmatrix},
    \qquad
    Q = \frac{1}{\beta}
    \begin{pmatrix}
        0 & -I_N \\
        I_N & 0
    \end{pmatrix}.
\end{align*}
Because the positional block of the diffusion matrix is identical zero and $\Gamma_{x_i} = 0$, the position equation simplifies to $dX_t = T_{\delta}(P_t) \, dt$. The resulting momentum friction is given by $-\bar{D}(p)T_{\delta}(p)$. A constant diffusion coefficient $\bar{D} = \gamma I_n$ would produce a bounded friction term $-\gamma T_{\delta}(p)$. To preserve linear damping $-\gamma p$, we instead make the choice
\begin{align*}
    M_{\delta}(p) &= (I_n + \delta p p^{\top})^{1/2},
    \qquad
    \bar{D}(p) = \gamma M_{\delta}(p),
    \
    \text{so that }M_{\delta}(p) T_{\delta}(p) = p.
\end{align*}

We define $\Gamma_{\delta}(p) = \operatorname{div}_p M_{\delta}(p)$ using the divergence induced by column major vectorization, where $\hat{p}_{i+n(j-1)} = p_{ij}$. The block structure $I_m \otimes M_{\delta}(p)$ yields the entries of $\Gamma_{\delta}$ as follows:
\begin{align}
    [\Gamma_{\delta}(p)]_{ij}
    &= \sum_{a=1}^{n}
    \frac{\partial [M_{\delta}(p)]_{ia}}{\partial p_{aj}},
    \qquad 1 \leq i \leq n, \quad 1 \leq j \leq m.
    \label{gammaDiv}
\end{align}
This relation is equivalently expressed as $\operatorname{div}_{\hat{p}}(I_m \otimes M_{\delta}(p)) = \operatorname{vec}(\Gamma_{\delta}(p))$. Substituting these specific choices into equation \eqref{mal_recipe} produces the stochastic system
\begin{align}
    dX_t &= T_{\delta}(P_t) \, dt,
    \nonumber\\
    dP_t &= \left[ -\nabla U(X_t) - \gamma P_t
    + \frac{\gamma}{\beta} \Gamma_{\delta}(P_t) \right] dt
    + \sqrt{\frac{2\gamma}{\beta}} M_{\delta}(P_t)^{1/2} \, dW_t,
    \label{mal_hamiltonian_SDE}
\end{align}
where $W_t$ represents an $n \times m$ matrix of independent standard Brownian motions. We observe that taking the limit as $\delta \to 0$ gives $K_{\delta}(p) \to |p|^2 / 2$, $T_{\delta}(p) \to p$, $M_{\delta}(p) \to I_n$, and $\Gamma_{\delta}(p) \to 0$. Thus, the limiting case successfully recovers the classical underdamped Langevin system
\[
    dX_t = P_t \, dt,
    \qquad
    dP_t = (-\nabla U(X_t) - \gamma P_t) \, dt
    + \sqrt{2\gamma\beta^{-1}} \, dW_t.
\]

Finally, For a time step $\lambda \in (0, \delta]$, the Euler Maruyama discretization of equation \eqref{mal_hamiltonian_SDE} is formulated as
\begin{align}
    X_{k+1} &= X_k + \lambda T_{\delta}(P_k),
    \nonumber\\
    P_{k+1} &= (1 - \gamma\lambda) P_k - \lambda \nabla U(X_k)
    + \frac{\lambda\gamma}{\beta} \Gamma_{\delta}(P_k)
    + \sqrt{\frac{2\lambda\gamma}{\beta}} M_{\delta}(P_k)^{1/2} \xi_{k+1},
    \label{mal_hamiltonian_MC}
\end{align}
\section{Assumptions \&  associated Remarks}
\label{sec:potential_assumptions}

Throughout, the Hessian is viewed as a self-adjoint operator,
with induced norm $|\cdot|_2:=\|\cdot\|_{\mathrm{F}\to\mathrm{F}}$.

\begin{assumption}\label{ass:U_smooth}
The potential $U:\mathbb{R}^{n\times m}\to\mathbb{R}$ belongs to
$C^{\infty}(\mathbb{R}^{n\times m};\mathbb{R})$.
\end{assumption}

\begin{assumption}\label{ass:U_curvature}
The Hessian satisfies the $L_0/L_1$ growth bound. There exist
$L_0,L_1\geq0$ such that
\begin{equation}
    \|\nabla^2U(x)\|_{\mathrm{F}\to\mathrm{F}}
    \leq L_0+L_1|\nabla U(x)|,
    \qquad \forall x\in\mathbb{R}^{n\times m}.
    \label{eq:U_curvature}
\end{equation}
\end{assumption}
\begin{remark}\label{rem:U_gradient_bounds}
Under Assumptions~\ref{ass:U_smooth} and~\ref{ass:U_curvature},
the gradient of $U$ is locally Lipschitz. 
\end{remark}
\begin{assumption}\label{ass:U_dissipativity}
The gradient is $2$-dissipative. There exist $w_1>0$ and
$w_2\geq0$ such that
\begin{equation}
    \langle x,\nabla U(x)\rangle
    \geq w_1|x|^2-w_2,
    \qquad \forall x\in\mathbb{R}^{n\times m}.
    \label{eq:U_dissipativity}
\end{equation}
\end{assumption}

\begin{remark}\label{rem:U_coercivity}
Under Assumptions~\ref{ass:U_smooth} and~\ref{ass:U_dissipativity},
there exists $C_U\geq0$ such that
\begin{equation}
    U(x)\geq\frac{w_1}{4}|x|^2-C_U,
    \qquad
    |\nabla U(x)|\geq w_1|x|-\frac{w_2}{|x|}
    \quad (x\neq0).
    \label{eq:U_coercive_bounds}
\end{equation}
Consequently, $U$ is coercive and $|\nabla U(x)|\to\infty$ as
$|x|\to\infty$.
\end{remark}

\begin{remark}\label{rem:U_normalization}
Since $U$ is continuous and coercive, it attains a global minimum.
Adding a constant changes neither the dynamics nor the normalized
Gibbs measure. Henceforth, without loss of generality, we set
$\min_xU(x)=0$, so that $U\geq0$.
Moreover, \eqref{eq:U_coercive_bounds} implies
\begin{equation}
    0<Z_\beta
    :=\int_{\mathbb{R}^{n\times m}}e^{-\beta U(x)}\,dx
    <\infty,
    \qquad \forall\beta>0.
    \label{eq:U_Gibbs_integrability}
\end{equation}
\end{remark}
\section{Main Results}
This first result is a consequence of Harris' theorem, see \citet{RobertsTweedie1996} and references therein.
\begin{condition} \label{harris_drift}\textit{(Geometric drift)}\\
    There exist a measurable $\mathcal{V}:\mathbb{R}^{2n\times m}\to[0,\infty)$ and constants $\varrho\in(0,1)$, $C_{\mathcal{V}}>0$ such that
    \begin{align}
        R\mathcal{V}\leq\varrho\mathcal{V}+C_{\mathcal{V}}.\label{harris_drift_ineq}
    \end{align}
\end{condition}
\begin{condition}\label{harris_minor}\textit{(Minorization)}\\
    For some $R_0>2C_{\mathcal{V}}/(1-\varrho)$, the sublevel set $C=\{\mathcal{V}\leq R_0\}$ satisfies
    \begin{align}
        R(z,\cdot)\geq\alpha_C\nu_C(\cdot),\ \ \forall z\in C,\label{harris_minor_ineq}
    \end{align}
    for some $\alpha_C\in(0,1)$ and a probability measure $\nu_C$.
\end{condition}
\noindent For $\theta\geq0$, the $(1+\theta\mathcal{V})$-total variation distance between probability measures $\mu_1,\mu_2$ is
\begin{align}
    \|\mu_1-\mu_2\|_{\mathrm{TV},\theta}=\sup_{|f|\leq1+\theta\mathcal{V}}|\mu_1(f)-\mu_2(f)|=\int(1+\theta\mathcal{V})\,d|\mu_1-\mu_2|,\label{weighted_TV}
\end{align}
which reduces to the total variation distance $\|\cdot\|_{\mathrm{TV}}$ for $\theta=0$.
\begin{theorem}\label{harris}
    Let Conditions \ref{harris_drift} and \ref{harris_minor} hold. Fix $\alpha_0\in(0,\alpha_C)$, $\varrho_0\in(\varrho+2C_{\mathcal{V}}/R_0,1)$ and set
    \begin{align*}
        \theta=\dfrac{\alpha_0}{C_{\mathcal{V}}},\ \ \bar{\alpha}=\max\left\{1-(\alpha_C-\alpha_0),\dfrac{2+R_0\theta\varrho_0}{2+R_0\theta}\right\}\in(0,1).
    \end{align*}
    Then for all probability measures $\mu_1,\mu_2$,
    \begin{align}
        \|\mu_1R-\mu_2R\|_{\mathrm{TV},\theta}\leq\bar{\alpha}\|\mu_1-\mu_2\|_{\mathrm{TV},\theta}.\label{harris_contraction}
    \end{align}
    Consequently, $R$ admits a unique invariant probability measure $\pi_R$, with $\pi_R(\mathcal{V})<\infty$, and for every $\mu_0$ with $\mu_0(\mathcal{V})<\infty$ and $n\geq0$,
    \begin{align}
        \|\mu_0R^n-\pi_R\|_{\mathrm{TV},\theta}\leq e^{-n\log\bar{\alpha}}\left(2+\theta\mu_0(\mathcal{V})+\theta\pi_R(\mathcal{V})\right).\label{harris_ergodic}
    \end{align}
\end{theorem}
\begin{proof}

Eq.\eqref{LypCond} gives Condition~1, with $\mathcal{V}=\Psi$, by Dynkin's formula. Moreover, Proposition \ref{minorization} implies Condition~2 for a suitable skeleton time. Hence both conditions of Theorem \ref{harris} are satisfied, and the conclusion follows from Theorems 1.2, 1.3 \& 3.2 in \citet{HairerMattingly2011}.
\end{proof}
\begin{corollary}\label{cor:SDE_moments}
Under the standing assumptions and the parameter restriction \eqref{b_restriction}, let $\Psi$ satisfy Eq.\eqref{LypCond}. Suppose that $\mathbb{E}[\Psi(X_0,P_0)]<\infty$. Then the SDE \eqref{mal_hamiltonian_SDE} admits a unique global strong solution. Moreover, for every $r>0$, there exists $C_r>0$, independent of time and of the initial law, such that
\begin{align}
    \sup_{t\geq0}\mathbb{E}\left[|X_t|^r+|P_t|^r\right]
    \leq C_r\left(\mathbb{E}[\Psi(X_0,P_0)]+\frac{C_{\mathcal{K}}}{\eta}\right)<\infty.
    \label{eq:SDE_uniform_moments}
\end{align}
\end{corollary}
\begin{proof}
Smoothness of $U$ and of the matrix functions defining the coefficients gives a unique local strong solution. By Eq.\eqref{eq:U_coercive_bounds}, Eq.\eqref{p_K_bound} and the parameter restriction \eqref{b_restriction}, $\Psi$ dominates every polynomial in $|x|$ and $|p|$ and has compact sublevel sets. Applying It\^o's formula to the stopped local solution and using Eq.\eqref{LypCond} yields that the process is non-explosive. Removing the stopping time yields
\[
    \mathbb{E}[\Psi(X_t,P_t)]
    \leq e^{-\eta t}\mathbb{E}[\Psi(X_0,P_0)]
    +\frac{C_{\mathcal{K}}}{\eta}(1-e^{-\eta t}),
    \qquad t\geq0.
\]
The asserted uniform polynomial moment bounds follow from the exponential domination by $\Psi$.
\end{proof}
\begin{theorem}\label{thm:discrete_uniform_moments}
Under the standing assumptions, fix $\delta\in(0,1]$ and $\lambda\in(0,\min\{\delta,\gamma^{-1}\}]$. Let $\mathcal{V}$ be defined by Eq.\eqref{mal_lyapunov}, and choose the hyperparameters so that
\[
    A_y<0,\qquad A_p+A_3<0.
\]
Suppose that $\mathbb{E}[\mathcal{V}(X_0,P_0)]<\infty$. Then, for every $r>0$, there exists $C_r>0$, independent of the number of iterations and of the initial conditions, such that the process \eqref{mal_hamiltonian_MC} satisfies
\begin{align}
    \sup_{n\geq0}\mathbb{E}\left[|X_n|^r+|P_n|^r\right]
    \leq C_r\left(1+\mathbb{E}[\mathcal{V}(X_0,P_0)]\right)<\infty.
    \label{eq:discrete_uniform_moments}
\end{align}
In particular, this conclusion holds for every deterministic initial state.
\end{theorem}
\begin{proof}
The discrete Lyapunov estimate gives contraction outside the compact sublevel set
\[
    \mathcal{C}=\left\{(x,p):K_{\delta}(p)+K_{\delta}\bigl(y(x,p)\bigr)\leq R_0\right\}.
\]
Boundedness of the one-step Lyapunov expectation on $\mathcal{C}$ therefore yields
\[
    \mathbb{E}[\mathcal{V}_{n+1}\mid\mathcal{F}_n]
    \leq e^{-\eta\lambda}\mathcal{V}_n+B_{\lambda}
\]
for some $B_{\lambda}<\infty$. Iterating this inequality bounds $\mathbb{E}[\mathcal{V}_n]$ uniformly in $n$.

The coercivity of $U$ controls all polynomial powers of $|x|$. Moreover, the $L_0/L_1$ condition bounds the growth of $\nabla U(x)$ exponentially in $|x|$, so every power of $|\nabla U(x)|$ is dominated by $e^{\eta U(x)}$. The linear coercivity of $K_{\delta}$ and the identity $p=y-\gamma^{-1}\nabla U(x)$ then show that $\mathcal{V}(x,p)$ dominates $|x|^r+|p|^r$ for every $r>0$. Taking expectations proves the asserted bounds.
\end{proof}

\section{Taming map Properties}
We start by citing a known result from singular-value calculus.
\begin{lemma}\label{lemMatrix1}\textit{(Corollary 2 \citet{Kittaneh1985})}\\
    Let $A,B$ be normal operators and let $f$ be a function defined on the union of the spectra $A$ and $B$. If $|f(x)-f(y)|\leq L|x-y|$ for all $x,y\in\sigma(A)\cup\sigma(B)$ and some positive constant $L$, then $|F(A)-F(B)|\leq L|A-B|$. Here $F(A)$ is understood as the map that applies $f(x)$ on the spectra of $A$, i.e. its singular-value calculus.
\end{lemma}
\noindent Since self-adjoint operators are normal, we can now propagate this result for the taming function,
\begin{lemma}\label{lemMatrix2}
    Let $f:\mathbb{R}\to\mathbb{R}$ be odd and $L-$Lipschitz with $F(A)=Uf(\Sigma)V^{\top}$ be its singular-value calculus. Then $|F(A)-F(B)|\leq L|A-B|$ for all $A,B\in\mathbb{R}^{m\times n}$.
\end{lemma}
\begin{proof}
    The symmetric matrix $\tilde{A}=[0,A;A^{\top},0]$ has eigenvalues $\pm\sigma_i(A)$ and $|n-m|$ zeroes [Exercise II.1.15, p35 in \citet{BhatiaMatrixAnalysis1997}]. Since $f$ is an odd function, one observes that $F(\tilde{A})=\tilde{F(A)}$. Hence Lemma \ref{lemMatrix1} yields
    \begin{align*}
        |F(A)-F(B)|=\dfrac{1}{\sqrt{2}}|\tilde{F(A)}-\tilde{F(B)}|=\dfrac{1}{\sqrt{2}}|F(\tilde{A})-F(\tilde{B})|\leq\dfrac{L}{\sqrt{2}}|\tilde{A}-\tilde{B}|=L|A-B|.
    \end{align*}
\end{proof}
\begin{remark}\label{SpectralProperties}
    The following properties hold for $T_{\lambda}$:
\end{remark}
\begin{enumerate}
    \item Since $\tau'(s)=(1+\lambda s^2)^{-3/2}\in(0,1]$, in view of Lemma \ref{lemMatrix2} one yields that $|T_{\lambda}(A)-T_{\lambda}(B)|\leq|A-B|$ for all $A,B\in\mathbb{R}^{m\times n}$.
    \item Since $T_{\lambda}$ is $1-$Lipschitz, $|\nabla T_{\lambda}|_{\text{F}\to \text{F}}\leq 1$. Indeed, we have $|\nabla T_{\lambda}(A)[B]|=\lim_{t\to 0}t^{-1}|T_{\lambda}(A+tB)-T_{\lambda}(A)|\leq\lim_{t\to 0}t^{-1}|tB|=|B|$.
    \item Since $T_{\lambda}$ is $1-$Lipschitz, one-sided Lipschitzness also follows, i.e. $\langle A-B,T_{\lambda}(A)-T_{\lambda}(B)\rangle\leq|A-B|^2$. Moreover, since $K_{\lambda}$ is convex, $\nabla K_{\lambda}=T_{\lambda}$ is co-coercive due to Baillon-Haddad Theorem [Corollary 18.17 pg 323 in H.Bauschke Convex Analysis and Monotone Operator Theory in Hilbert Spaces]. In particular it holds that, $\langle A-B,T_{\lambda}(A)-T_{\lambda}(B)\rangle\geq|T_{\lambda}(A)-T_{\lambda}(B)|^2$.
    \item The elementary bound $\sqrt{1+\lambda s^2}\geq\sqrt{\lambda}|s|$ gives $K_{\lambda}(p)\geq\frac{1}{\lambda}(|p|-d)$, hence $e^{-\beta K_{\lambda}}\in L^1,\ \forall \beta>0$.
    \item One calculates $\tau(s)=\lambda^{-1/2}\sqrt{\lambda s^2} (1+\lambda s^2)^{-1/2}\leq \lambda^{-1/2}$. Since the Frobenius norm is unitary invariant, for $p=U\Sigma V^{\top}$ one gets, $|T_{\lambda}(p)|=|T_{\lambda}(\Sigma)|=\left(\sum\tau^2(\sigma_i(p))\right)^{1/2}\leq \sqrt{d\lambda^{-1}}$. 
    \item Next we control the difference between the tamed and untamed momentum.
    \begin{align*}
        |s-\tau(s)|&\leq |s|\left |1-\dfrac{1}{\sqrt{1+\lambda s^2}}\right|=|s|\left |\dfrac{\sqrt{1+\lambda s^2}-1}{\sqrt{1+\lambda s^2}}\right|=|s|\left|\dfrac{\lambda s^2}{\sqrt{1+\lambda s^2}(\sqrt{1+\lambda s^2}+1)} \right|\leq \lambda|s|^3\\
        &\leq |s|\left |\dfrac{\sqrt{\lambda s^2}}{\sqrt{1+\lambda s^2}}\right|\leq \sqrt{\lambda}s^2.
    \end{align*}
    Through Frobenius's unitary invariance we obtain
    \begin{align*}
        |p-T_{\lambda}(p)|^2&=|U\left(\Sigma-T_{\lambda}(\Sigma)\right)V^{\top}|^2= \sum_{i\leq d}|\sigma_i(p)-\tau(\sigma_i(p))|^2\leq \lambda\sum_{i\leq d}\sigma^4_i(p)\\
        &\leq \lambda |p^{\top}p|^2\leq \lambda|p|^4.
    \end{align*}
    \item Lastly we show that $T_{\lambda}$ is $1-$dissipative. Since $\sqrt{1+\lambda s^2}\leq 1+\sqrt{\lambda} s$, we have
    \begin{align*}
        \sum_{i\leq d}\sigma_i(p)\tau(\sigma_i(p))=\sum_{i\leq d}\dfrac{1}{\lambda}\dfrac{\lambda \sigma_i(p)^2}{\sqrt{1+\lambda \sigma_i(p)^2}}\geq \sum_{i\leq d}\dfrac{1}{\lambda}\dfrac{\lambda \sigma_i(p)^2}{1+\sqrt{\lambda} \sigma_i(p)}\geq \sum_{i\leq d}\dfrac{1}{\lambda}\left(\sqrt{\lambda}\sigma_i(p)-1\right).
    \end{align*}
    Therefore, $\langle p,T_{\lambda}(p)\rangle=\text{Tr}(p^{\top}T_{\lambda}(p))=\text{Tr}(\Sigma^{\top}\tau(\Sigma))=\langle \Sigma,\tau(\Sigma)\rangle\geq \lambda^{-1/2}\sum_{i\leq d}\sigma_i(p)-d\lambda^{-1}.$ This yields
    \begin{align*}
        \dfrac{1}{|p|}\langle p,T_{\lambda}(p)\rangle -\dfrac{2\lambda}{|p|}|T_{\lambda}(p)|^2\geq \dfrac{|p|_*}{\lambda^{1/2}|p|}-\dfrac{d}{\lambda|p|}-\dfrac{2\lambda}{|p|}\dfrac{d}{\lambda}\geq\dfrac{1}{\lambda^{1/2}}-\dfrac{1}{|p|}(\lambda^{-1}-2d).
    \end{align*}
    Hence we obtain
    \begin{align}
        \liminf_{|p|\to\infty}\left\langle \dfrac{p}{|p|},T_{\lambda}(p)\right\rangle-\dfrac{2\lambda}{|p|}\left|T_{\lambda}(p) \right|^2\geq \lambda^{-1/2}>0\label{kinDiss}
    \end{align}
\end{enumerate}
\begin{lemma}\label{CharConvex}\textit{(Characterization of Convexity, \citet{Lewis1995})}\\
Let $A\in\mathbb{R}^{m\times n}$ and $\sigma:\mathbb{R}^{m\times n}\to\mathbb{R}^d$ (with $d=\min(m,n)$) have components the singular values of $A$. Let $f:\mathbb{R}^d\to(-\infty,+\infty]$ be absolutely symmetric. The corresponding unitarily invariant function $f(\sigma)$ is convex on $\mathbb{R}^{m\times n}$ if and only if $f$ is convex.
\end{lemma}
\begin{lemma}\label{HessianBound}
    For every $p\in\mathbb{R}^{m\times n}$, with $\sigma_1(p)$ denoting the largest spectral value of $p$, the following bound holds
    \begin{align}
        0\prec(1+\lambda\sigma^2_1(p))^{-3/2}I\preceq\nabla^2K_{\lambda}(p)\preceq I
    \end{align}
\end{lemma}
\begin{proof}
    We define the auxiliary function $F(p)=\frac{1}{2}|p|^2-K_{\lambda}(p)$. In view of \eqref{kinEnergy} we obtain the representation
    \begin{align*}
        F(p)=\sum_{i\leq d}\dfrac{\sigma^2_i(p)}{2}-k(\sigma_i(p))=\sum_{i\leq d}f(\sigma_i(p)), \ f(s):=\dfrac{s^2}{2}-\dfrac{1}{\lambda}\left(\sqrt{1+\lambda s^2}-1\right).
    \end{align*}
    We have $\forall s\in\mathbb{R}$ that $f'(s)=s-s(1+\lambda s^2)^{-1/2}\Rightarrow f''(s)=1-(1+\lambda s^2)^{-3/2}\geq 0$, thus $f$ is a convex function and symmetric by construction, hence Lemma \ref{CharConvex} implies the composition $f(\sigma)$ is also convex. Since $F$ is defined as the sum of convex functions, it follows immediately that $F$ is convex. Hence $\nabla^2 F(p)\succeq0\Leftrightarrow I-\nabla^2 K_{\lambda}(p)\succeq0$. Define the positive constant $c=(1+\lambda R^2)^{-3/2}$, such that $R>|p|_2=\sigma_1(p)$. Then one defines the auxiliary function 
    \begin{align*}
       G(p)=\sum_{i\leq d}g(\sigma_i(p)),\ g(s)=\begin{cases}
            k(s)-\frac{c}{2}s^2,\ &|s|\leq R\\
            +\infty ,\ &|s|>R
        \end{cases}
    \end{align*}
    and concludes $G$ is convex. Since $G(p)=K_{\lambda}(p)-\frac{c}{2}|p|^2$ on $\{\sigma_1(p)\leq R\}$ and $R$ was chosen to be arbitrary, one recovers the statement.
\end{proof}
\begin{remark}
    Additionally, Lemma \ref{HessianBound} provides us the following bound on the Laplacian $\Delta K_{\lambda}=\text{Tr}(\nabla^2K_{\lambda})\leq mn$.
\end{remark}
\begin{definition}
    Let $f\in C^1$. Then we denote by $f^{[1]}$ the first divided difference of $f$. See Section V.3 in \citet{BhatiaMatrixAnalysis1997} for further details.
    \begin{align*}
        f^{[1]}(x,y)=\dfrac{f(x)-f(y)}{x-y},\ \text{where } f^{[1]}(x,x):=f'(x).
    \end{align*}
\end{definition}
\begin{lemma}\label{boundLaplacian}
    For every $p \in \mathbb{R}^{m \times n}$ the following bound holds
    \begin{align}
        \left|\Delta_p T_\lambda(p)\right| \leq (m+n)d^{1/2}\lambda^{1/2}.
    \end{align}
\end{lemma}
\begin{proof}
Since $\Kl$ is smooth, partial derivatives commute, hence $\Delta_p\Tl=\dpp\Delta_p\Kl$. Moreover matrices with pairwise distinct, positive singular values are dense, thus by continuity it suffices to consider such $p$ without loss of generality. Let $p=U\Sigma V^{\top}$, the Hessian of the spectral map \eqref{taming} acting on $A\in\mathbb{R}^{m\times n}$ admits the closed form [Theorem 2.2. in Christian Erwin Kümmerle, Understanding and Enhancing Data Recovery Algorithms, TU Munchen 2019]
\begin{align*}
    \nabla_p ^2K_{\lambda}(p)[A]=U\left(\begin{matrix}
    H_1\odot S(A_1)+H_2\odot S^*(A_1)\\
    H_3\odot A_2
    \end{matrix}\right)V^{\top},
\end{align*}
where $UAV^{\top}=[A_1;A_2]$ such that $A_1\in\mathbb{R}^{n\times n}$. Additionally one defines the symmetrization and anti-symmetrization operators, $S(B)=(1/2)(B+B^{\top})$ and $S^*(B)=(1/2)(B-B^{\top})$ respectively.
\begin{align*}
    H_1(i,j)=\tau^{[1]}(\sigma_i,\sigma_j), H_2(i,j)=\tau^{[1]}(\sigma_i,-\sigma_j), \text{and } H_3(i,j)=\tau^{[1]}(\sigma_j,0) 
\end{align*}
Now taking the trace, one recovers the Laplacian in the following form
\begin{align}
    \Delta_p K_{\lambda}(p)&=\sum_{i=1}^nH_1(i,i)+\sum_{1\leq i<j\leq n}\left(H_1(i,j)+H_2(i,j)\right)+\sum_{i=1}^{m-n}\sum_{j=1}^n H_3(i,j)\nonumber\\
&=\sum_{i\leq d}\tau'(\sigma_i)
    +\sum_{i<j\leq d}\left[\tau^{[1]}(\sigma_i,\sigma_j)+\tau^{[1]}(\sigma_i,-\sigma_j)\right]
    +(m-d)\sum_{i\leq d}\tau^{[1]}(\sigma_i,0)\label{laplaceSVD}
\end{align}
Notice how $\Delta_pK_{\lambda(p)}$ is essentially a smooth function of the singular values, namely $C(\sigma_1,\ldots,\sigma_d)$. Moreover it is symmetric and even in each argument, since $\tau'$ is even and, the oddness of $\tau$ gives $\tau^{[1]}(-x,y)=\tau^{[1]}(x,-y)$. Therefore the gradient formula for spectral functions [Theorem 3.1 \citet{Lewis1995}] applies, so that
\begin{align*}
    \dpp\Delta_p\Kl(p)=\nabla_pC(\sigma_1(p),\ldots,\sigma_d(p))=U\left(\text{Diag}\nabla_{\sigma} C(\sigma_1(p),\ldots,\sigma_d(p))\right)V^{\top}.
\end{align*}
Thus by the unitary invariance of the Frobenius norm one obtains
\begin{align*}
    |\Delta_p\Tl(p)|=|\dpp\Delta_p\Kl(p)|=\Big(\sum_{i=1}^d(\partial_{\sigma_i}C)^2\Big)^{1/2}
    \leq\sqrt{d}\,\max_{1\leq i\leq d}|\partial_{\sigma_i}C|.
\end{align*}
Since $\tau''(x)=-3\lambda x(1+\lambda x^2)^{-5/2}$, one has the elementary bound $|\tau''(x)|\leq\sqrt{\lambda}$ for all $x\in\mathbb{R}$. Differentiating the finite first difference once gives $\partial_x\tau^{[1]}(x,y)=(\tau'(x)-\tau^{[1]}(x,y))(x-y)^{-1}$. Next, one applies Taylor's theorem to obtain $\tau^{[1]}(x,y)=\tau'(x)+\frac{1}{2}\tau''(\xi)(y-x)$ for some $\xi\in(x,y)$ without the loss of generalization. Therefore we have that $\partial_x\tau^{[1]}(x,y)=\frac{1}{2}\tau''(\xi)$. Collecting the off-diagonal terms in \eqref{laplaceSVD} yields
\begin{align*}
    |\partial_{\sigma_i}C|\leq \sqrt{\lambda}+2(n-1)\dfrac{\sqrt{\lambda}}{2}+(m-n)\dfrac{\sqrt{\lambda}}{2}=\dfrac{m+n}{2}\sqrt{\lambda}.
\end{align*}
and one concludes with the statement of the Lemma.
\end{proof}
\section{Geometric Drift Condition}
Define the concatenation of the latter system, with joint variable $z=(x,p)$, drift and diffusion coefficients, $F$ and $D$ respectively. Applying Ito's formula on a sufficiently smooth function $f$ yields the usual infinitesimal generator
\begin{align}
    \mathcal{A}f&=\langle F,\nabla_z f\rangle+\Tr(D\nabla^2_zf)\nonumber\\
    &=\langle T_{\delta}(p),\nabla_xf\rangle-\left\langle  \nabla_x U(x)+\gamma p-\dfrac{\gamma}{\beta}\Gamma_{\delta}(p),\nabla_p f\right\rangle+\dfrac{\gamma}{\beta}\sum_{j=1}^m\sum_{i,k=1}^n[M_{\delta}(p)]_{ik}\dfrac{\partial^2f}{\partial p_{ij}\partial p_{kj}}.\label{s3_gen}
\end{align}
Moreover, applying the generator to $e^f$, the chain rule gives
\begin{align}
    \dfrac{\mathcal{A}e^f}{e^f}=\mathcal{A}f+\dfrac{\gamma}{\beta}\langle\nabla_pf,M_{\delta}(p)\nabla_pf\rangle.\label{s3_exp}
\end{align}
\noindent We adapt the methodology introduced by Wu in \cite{Wu2001}. Although our setting involves unbounded multiplicative noise and generic kinetic energy, Wu's underlying approach remains applicable. Motivated by the arguments in \cite[Theorems 3.1 and 4.1]{Wu2001}, we require a function $G:\mathbb{R}^{m\times n}\to\mathbb{R}^{m\times n}$ such that $|G(x)|\leq 1$, $\langle \nabla U(x),G(x)\rangle\to\infty$ as $|x|\to\infty$, and $\sup|Q_{\delta}^{\top}(p)G(x)|<\infty$. Here, $Q_{\delta}(p)=M_{\delta}^{-1}(p)=(I+\delta pp^{\top})^{-1/2}$ denotes the preconditioning matrix of $p$ i.e., $T_{\delta}(p)=Q_{\delta}(p)p$. To satisfy these criteria, we define
\begin{align*}
    G(x)=\frac{\nabla U(x)}{\sqrt{1+|\nabla U(x)|^2}},
\end{align*}
which immediately guarantees that
\begin{align*}
    |G(x)|\leq1 \quad \text{and} \quad \langle\nabla U(x),G(x)\rangle=\frac{|\nabla U(x)|^2}{\sqrt{1+|\nabla U(x)|^2}}.
\end{align*}

It is worth noting that Hypothesis 2 in \cite{Wu2001} fails in our context, as there is no constant $c>0$ such that $Q_{\delta}(p)\succeq cI$ uniformly in $p$. However, because our Hamiltonian includes the tamed kinetic energy rather than the standard quadratic one, this lack of uniform positivity does not restrict our analysis.
For $a\in(0,1)$ and $b>0$, let us define the Lyapunov function
\begin{align}
    \Psi(x,p)=\exp\left\{a\beta H(x,p)+b\langle G(x),p\rangle+\dfrac{bd}{\sqrt{\delta}}\right\},\label{s3_Psi}
\end{align}
where the constant $bd/\sqrt{\delta}$ is included to ensure that $\Psi\geq1$.
\begin{lemma}\label{lem_master_equation}
Let $\mathcal{A}$ be the generator in \eqref{s3_gen}, and let
$\Psi$ be the Lyapunov function defined in \eqref{s3_Psi}, with
$a\in(0,1)$ and $b>0$. Define
\[
    \phi(x,p):=-\frac{\mathcal{A}\Psi(x,p)}{\Psi(x,p)}.
\]
Then
\begin{align}
    \phi(x,p)
    &=a(1-a)\beta\gamma\langle T_{\delta}(p),p\rangle
      +(1-2a)b\gamma\langle G(x),p\rangle\nonumber\\
    &\quad-b\langle T_{\delta}(p),\mathcal{J}_G(x)^\top p\rangle
      +b\langle\nabla U(x),G(x)\rangle-a\gamma nm\nonumber\\
    &\quad-b\gamma\beta^{-1}\langle\Gamma_{\delta}(p),G(x)\rangle
      -b^2\gamma\beta^{-1}
      \langle G(x),M_{\delta}(p)G(x)\rangle,
      \label{master1}
\end{align}
where $\mathcal{J}_G(x)^\top$ denotes the adjoint of the derivative
of $G$ with respect to the Frobenius inner product.
\end{lemma}
\begin{proof}
Let $\psi=\log\Psi$, by direction calculations we get 
\begin{align}
    \nabla_x\psi&=a\beta\nabla_xH+b\nabla_x\langle G(x),p\rangle=a\beta\nabla U(x)+b\mathcal{J}_G^\top p,\nonumber\\
    \nabla_p\psi&=a\beta\nabla_pH+b\nabla_p\langle G(x),p\rangle=a\beta T_{\lambda}(p)+bG(x),\label{calc1}\\
    \Delta_p\psi&=a\beta\Tr(\nabla_p^2H)=a\beta\Tr(\mathcal{J}_{T_{\delta}})=a\beta \text{div}(T_{\delta}),\nonumber
\end{align}
where $[\mathcal{J}_G^\top p]_{k\ell}=\sum_{i=1}^n\sum_{j=1}^m\frac{[\partial G(x)]_{ij}}{\partial x_{k\ell}}p_{ij}$ and $\text{div}(T_{\delta})=\sum_{i=1}^n\sum_{j=1}^m\frac{\partial[T_{\delta}(p)]_{ij}}{\partial p_{ij}}$. Moreover, we have
\begin{align}
    \sum_{j=1}^m\sum_{i,k=1}^n[M_{\delta}(p)]_{ik}\dfrac{\partial^2K_{\delta}(p)}{\partial p_{ij}\partial p_{kj}}=\sum_{j=1}^m\sum_{i,k=1}^n[M_{\delta}(p)]_{ik}\dfrac{\partial[ T_{\delta}(p)]_{kj}}{\partial p_{ij}}.\nonumber
\end{align}
Since we know that $A\frac{\partial B}{\partial p_{ij}}=\frac{\partial}{\partial p_{ij}}(AB)-\frac{\partial A}{\partial p_{ij}}B$, we rewrite the latter term inside the sum
\begin{align}
    \sum_{j=1}^m\sum_{i,k=1}^n\left(\dfrac{\partial}{\partial p_{ij}}\left([M_{\delta}(p)]_{ik}[T_{\delta}(p)]_{kj}\right)-\dfrac{\partial[M_{\delta}(p)]_{ik}}{\partial p_{ij}}[T_{\delta}(p)]_{kj}\right).\nonumber
\end{align}
Recall the identity $M_{\delta}(p)T_{\delta}(p)=p$. One notices that the sum over $k$ in the first term corresponds to the matrix multiplication of $M_{\delta},T_{\delta}$. Hence we obtain
\begin{align}
    \sum_{j=1}^m\sum_{i=1}^n&\dfrac{\partial p_{ij}}{\partial p_{ij}}-\sum_{j=1}^m\sum_{i,k=1}^n\dfrac{\partial[M_{\delta}(p)]_{ik}}{\partial p_{ij}}[T_{\delta}(p)]_{kj}\nonumber\\
    =&nm-\sum_{j=1}^m\sum_{i,k=1}^n\dfrac{\partial[M_{\delta}(p)]_{ik}}{\partial p_{ij}}[T_{\delta}(p)]_{kj}.\nonumber
\end{align}
Moreover, by definition we have $[\Gamma_{\delta}(p)]_{kj}:=\sum_{i=1}^n\frac{\partial[M_{\delta}(p)]_{ki}}{\partial p_{ij}}$. Using the fact that $M_{\delta}$ is a symmetric matrix, we further simplify to
\begin{align}
    nm-\sum_{j=1}^m\sum_{k=1}^n[\Gamma_{\delta}(p)]_{kj}[T_{\delta}(p)]_{kj}=nm-\langle\Gamma_{\delta}(p),T_{\delta}(p)\rangle.\label{calc00}
\end{align}
Applying the generator of Eq.\eqref{s3_gen} to $\psi$ and substituting from \eqref{calc1} and \eqref{calc00} yields
\begin{align}
    \mathcal{L}\psi&=a\beta\langle T_{\lambda}(p),\nabla U(x)\rangle+b\langle T_{\lambda}(p),\mathcal{J}_G^\top p\rangle-a\beta\langle \nabla U(x),T_{\lambda}(p)\rangle-b\langle \nabla U(x), G(x) \rangle\nonumber\\
    &-a\beta\gamma\langle p,T_{\lambda}(p)\rangle-b\gamma\langle p, G(x) \rangle+a\gamma\langle \Gamma_{\delta}(p),T_{\lambda}(p)\rangle+b\gamma\beta^{-1}\langle \Gamma_{\delta}(p), G(x) \rangle\nonumber\\
    &+a\gamma nm-a\gamma\langle\Gamma_{\delta}(p),T_{\delta}(p)\rangle.\nonumber
\end{align}
The terms $\pm\langle\Gamma_{\delta}(p),T_{\delta}(p)\rangle$ and $\pm\langle T_{\lambda}(p),\nabla U(x)\rangle$ cancel out, leading to
\begin{align}
    \mathcal{L}\psi=&a\gamma nm+b\langle T_{\lambda}(p),\mathcal{J}_G^\top p\rangle-b\langle \nabla U(x), G(x) \rangle\nonumber\\
    &-a\beta\gamma\langle p,T_{\lambda}(p)\rangle-b\gamma\langle p, G(x) \rangle+b\gamma\beta^{-1}\langle \Gamma_{\delta}(p), G(x) \rangle.\label{calc2}
\end{align}
Moreover, in view of $M_{\delta}(p)T_{\delta}(p)=p$ and the identity
\begin{align}
    \sum_{j=1}^m\sum_{i,k=1}^n[M_{\delta}(p)]_{ik}\dfrac{\partial f}{\partial p_{ij}}\dfrac{\partial f}{\partial p_{kj}}=\langle \nabla_pf,M_{\delta}(p)\nabla_p f\rangle,\nonumber
\end{align}
one obtains the following equivalence for $f=H(x,p)$
\begin{align}
    \sum_{j=1}^m\sum_{i,k=1}^n[M_{\delta}(p)]_{ik}\dfrac{\partial H}{\partial p_{ij}}\dfrac{\partial H}{\partial p_{kj}}=\langle T_{\delta}(p),p\rangle.\label{calc3}
\end{align}
Therefore through Eq.\eqref{calc2}-\eqref{calc3}, we recover the master equation
\begin{align}
    \phi&=-\dfrac{\mathcal{A}\Psi}{\Psi}=-\mathcal{A}\psi-\dfrac{\gamma}{\beta}\langle\nabla_p \psi,M_{\delta}(p)\nabla_p \psi\rangle\nonumber\\
    &=-a\gamma nm-b\langle T_{\lambda}(p),\mathcal{J}_G^\top p\rangle+b\langle \nabla U(x), G(x) \rangle\nonumber\\
    &+a\beta\gamma\langle p,T_{\lambda}(p)\rangle+b\gamma\langle p, G(x) \rangle-b\gamma\beta^{-1}\langle \Gamma_{\delta}(p), G(x) \rangle\nonumber\\
    &-a^2\beta\gamma\langle T_{\delta}(p),p\rangle-2ab\gamma\langle G(x) ,p\rangle-b^2\gamma\beta^{-1}\langle G(x),M_{\delta}(p)G(x)\rangle.\nonumber
\end{align}
Rearranging the latter equation proves Eq.\eqref{master1}.
\end{proof}
\begin{remark}
By setting $b=0$ in \eqref{calc2}--\eqref{calc3}, we obtain the energy identities
\begin{align}
    \mathcal{A}H
    &=\gamma\left(\frac{nm}{\beta}
      -\langle p,T_\delta(p)\rangle\right), \nonumber\\
    \mathcal{A}e^{a\beta H}
    &=a\gamma\left(nm-(1-a)\beta
      \langle p,T_\delta(p)\rangle\right)e^{a\beta H}.
\end{align}
Observe that for $0<a<1$, the drift of the exponential energy contains the dissipative term $-a(1-a)\beta\gamma\langle p,T_\delta(p)\rangle$. However, the value $a=1$ acts as a critical threshold, in the sense that the dissipative term vanishes, reducing the equation to
\[
    \mathcal{A}e^{\beta H}=\gamma nm\,e^{\beta H}.
\]
Consequently, $a=1$ marks the upper bound at which one can expect to establish uniform-in-time exponential moment bounds.
\end{remark}

\noindent Throughout this section we write $d=\min\{m,n\}$ for
the number of singular values. Recall that
\[
    G(x)=\frac{\nabla U(x)}{\sqrt{1+|\nabla U(x)|^2}},
    \qquad |G(x)|\leq1.
\]
Under Assumption \ref{ass:U_curvature}, we have
\begin{align*}
    |\mathcal{J}_G(x)|_{\mathrm{F}\to\mathrm{F}}
    \leq
    \frac{|\nabla^2U(x)|_{\mathrm{F}\to\mathrm{F}}}
    {\sqrt{1+|\nabla U(x)|^2}}
    \leq L_0+L_1.
\end{align*}
Define
\[
    S_{\delta}(p):=\langle T_{\delta}(p),p\rangle.
\]
The spectral bounds give
\begin{align}
    S_{\delta}(p)
    &\geq\frac{|p|}{\sqrt{\delta}}-\frac{1}{\delta},
    \qquad
    |p|\leq\sqrt{\delta}S_{\delta}(p)+\delta^{-1/2},
    \label{S_delta_coercive}\\
    |T_{\delta}(p)||p|
    &\leq
    \frac{\sqrt{d}\sigma_1(p)|p|}
    {\sqrt{1+\delta\sigma_1^2(p)}}
    \leq
    \frac{\sqrt{d}|p|^2}
    {\sqrt{1+\delta\sigma_1^2(p)}}
    \leq\sqrt{d}S_{\delta}(p).
    \label{Tp_product_bound}
\end{align}
We also use the divergence bound
$|\Gamma_{\delta}(p)|^2\leq n^2d\delta$ and the estimate
\begin{align*}
    |M_{\delta}(p)|_2
    &=\sqrt{1+\delta\sigma_1^2(p)}\\
    &\leq
    1+\frac{\delta\sigma_1^2(p)}
    {\sqrt{1+\delta\sigma_1^2(p)}}
    \leq1+\delta S_{\delta}(p).
\end{align*}

\noindent Consequently, the master equation \eqref{master1} yields
\begin{align}
    \phi
    &\geq a(1-a)\beta\gamma S_{\delta}(p)
    -|1-2a|b\gamma|p|
    -b(L_0+L_1)|T_{\delta}(p)||p|\nonumber\\
    &\quad
    +b\frac{|\nabla U(x)|^2}{\sqrt{1+|\nabla U(x)|^2}}
    -a\gamma nm
    -\frac{b\gamma}{\beta}|\Gamma_{\delta}(p)|
    -\frac{b^2\gamma}{\beta}|M_{\delta}(p)|_2.
    \label{phi_bound1}
\end{align}
Collecting these estimates gives
\begin{align}
    \phi
    &\geq
    \left[a(1-a)\beta\gamma-bB_{\delta}
    -\frac{b^2\gamma\delta}{\beta}\right]S_{\delta}(p)
    +b\frac{|\nabla U(x)|^2}{\sqrt{1+|\nabla U(x)|^2}}
    -\gamma J,
    \label{phi_bound2}\\
    B_{\delta}
    &:=\sqrt{d}(L_0+L_1)+|1-2a|\gamma\sqrt{\delta},
    \nonumber\\
    J
    &:=anm+\frac{|1-2a|b}{\sqrt{\delta}}
    +\frac{bn\sqrt{d\delta}}{\beta}
    +\frac{b^2}{\beta}.
    \nonumber
\end{align}

\begin{lemma}[Lyapunov drift condition]\label{lem_lyapunov_drift}
Suppose that Assumptions \ref{ass:U_smooth},
\ref{ass:U_curvature} and \ref{ass:U_dissipativity} hold,
and normalize $U$ so that $U\geq0$.
Let $\Psi$ be defined by \eqref{s3_Psi}, and fix
$\delta,\beta,\gamma>0$ and $a,q\in(0,1)$.
Choose $b>0$ such that
\begin{align}
    b\leq\min\left\{
    \frac{a\beta}{2\sqrt{\delta}},\,
    \frac{qa(1-a)\beta\gamma}{2B_{\delta}},\,
    \beta\sqrt{\frac{qa(1-a)}{2\delta}}
    \right\}.
    \label{b_restriction}
\end{align}
If $B_{\delta}=0$, the middle restriction is omitted.

For any $\eta>0$, define
\begin{align*}
    \Lambda_q&:=(1-q)a(1-a)\beta\gamma,\\
    R&:=\frac{\gamma J+\eta}{\Lambda_q},
\end{align*}
and
\begin{align}
    y:=\max\left\{1,\frac{\sqrt{2}(\gamma J+\eta)}{b}\right\},
    \qquad
    r:=\frac{y+\sqrt{y^2+4w_1w_2}}{2w_1}.
    \label{x_radius}
\end{align}
Let
\[
    \mathcal{K}
    :=\{(x,p):S_{\delta}(p)\leq R,\ |x|\leq r\},
    \qquad
    U_r:=\sup_{|x|\leq r}U(x).
\]
Then $\Psi\geq1$, the set $\mathcal{K}$ is compact, and
\begin{align}
    \mathcal{A}\Psi
    \leq-\eta\Psi+C_{\mathcal{K}}\mathbb{I}_{\mathcal{K}},
    \label{LypCond}
\end{align}
where
\[
    C_{\mathcal{K}}
    :=(\eta+\gamma J)\exp\left(
    a\beta U_r+(a\beta+b\sqrt{\delta})R
    +\frac{2bd}{\sqrt{\delta}}\right).
\]
In particular, the drift condition holds for any fixed
$\delta,\beta,\gamma>0$ after choosing $b>0$ sufficiently small.
The constants and the compact set may depend on $\delta$.
\end{lemma}

\begin{proof}
In view of \eqref{b_restriction}, we have
\begin{align*}
    bB_{\delta}+\dfrac{b^2\gamma\delta}{\beta}
    \leq qa(1-a)\beta\gamma.
\end{align*}
Therefore, setting
$\Lambda_q=(1-q)a(1-a)\beta\gamma>0$, we obtain
\begin{align}
    \phi\geq\Lambda_qS_{\delta}(p)
    +b\dfrac{|\nabla U(x)|^2}{\sqrt{1+|\nabla U(x)|^2}}
    -\gamma J.
    \label{phi_bound_coercive}
\end{align}

\noindent Moreover, the first restriction in \eqref{b_restriction}
ensures that the normalization in \eqref{s3_Psi} is valid. Indeed,
\begin{align}
    |p|\leq\sum_{i=1}^d\sigma_i(p)
    \leq\sqrt{\delta}K_{\delta}(p)
    +\dfrac{d}{\sqrt{\delta}},
    \label{p_K_bound}
\end{align}
and hence
\begin{align*}
    \log\Psi
    &=a\beta U(x)+a\beta K_{\delta}(p)
    +b\langle G(x),p\rangle+\dfrac{bd}{\sqrt{\delta}}\\
    &\geq a\beta U(x)
    +(a\beta-b\sqrt{\delta})K_{\delta}(p)
    \geq0.
\end{align*}
Consequently, $\Psi\geq1$.

\noindent\textbf{Domain (I)} Fix $\eta>0$ and introduce
\begin{align*}
    R=\dfrac{\gamma J+\eta}{\Lambda_q},
    \qquad
    \Omega=\{p\in\mathbb{R}^{n\times m}:S_{\delta}(p)\geq R\}.
\end{align*}
If $p\in\Omega$, dropping the second non-negative term in
\eqref{phi_bound_coercive} yields
\begin{align}
    \phi\geq\Lambda_qR-\gamma J=\eta>0.
    \label{phi_bound3}
\end{align}
Notice that the threshold $R$ is chosen according to the constants
in the drift inequality. In view of \eqref{S_delta_coercive},
this is possible for any fixed $\delta>0$, without imposing
a smallness condition on $\delta$.

\noindent\textbf{Domain (II)} Suppose that $p\notin\Omega$.
Due to the dissipativity of $U(x)$, for $|x|>0$ we have
\begin{align*}
    |x||\nabla U(x)|
    &\geq\langle x,\nabla U(x)\rangle
    \geq w_1|x|^2-w_2,\\
    |\nabla U(x)|
    &\geq w_1|x|-\dfrac{w_2}{|x|}.
\end{align*}
Recall that
\begin{align*}
    y=\max\left\{1,\dfrac{\sqrt{2}(\gamma J+\eta)}{b}\right\},
    \qquad
    r=\dfrac{y+\sqrt{y^2+4w_1w_2}}{2w_1}.
\end{align*}
Then $|\nabla U(x)|\geq y\geq1$ whenever $|x|\geq r$.
Since $t^2/\sqrt{1+t^2}\geq t/\sqrt{2}$ for all $t\geq1$,
dropping the first non-negative term in
\eqref{phi_bound_coercive} yields
\begin{align}
    \phi
    \geq\dfrac{b}{\sqrt{2}}|\nabla U(x)|-\gamma J
    \geq\dfrac{by}{\sqrt{2}}-\gamma J
    \geq\eta>0.
    \label{phi_bound4}
\end{align}

\noindent\textbf{Domain (III)} Finally, when $p\notin\Omega$
and $|x|\leq r$, we land on the compact set
\begin{align*}
    \mathcal{K}
    =\{(x,p):S_{\delta}(p)\leq R,\ |x|\leq r\}.
\end{align*}
Indeed, \eqref{S_delta_coercive} gives
$|p|\leq\sqrt{\delta}R+\delta^{-1/2}$ on $\mathcal{K}$.
From \eqref{phi_bound_coercive}, one also obtains
$\phi\geq-\gamma J$. Combining the three domains,
we therefore get
\begin{align}
    \mathcal{A}\Psi
    \leq-\eta\Psi+(\eta+\gamma J)
    \left(\sup_{\mathcal{K}}\Psi\right)
    \mathbb{I}_{\mathcal{K}}.
    \label{drift_before_constant}
\end{align}

\noindent To quantify the constant, write
$U_r=\sup_{|x|\leq r}U(x)<\infty$. Notice that
\begin{align*}
    K_{\delta}(p)
    &=\dfrac{1}{\delta}\sum_{i=1}^d
    \left(\sqrt{1+\delta\sigma_i^2(p)}-1\right)\\
    &\leq\sum_{i=1}^d
    \dfrac{\sigma_i^2(p)}{\sqrt{1+\delta\sigma_i^2(p)}}
    =S_{\delta}(p).
\end{align*}
Consequently, using \eqref{p_K_bound}, for
$(x,p)\in\mathcal{K}$ we have
\begin{align*}
    \log\Psi(x,p)
    &\leq a\beta U_r+a\beta K_{\delta}(p)
    +b|p|+\dfrac{bd}{\sqrt{\delta}}\\
    &\leq a\beta U_r+(a\beta+b\sqrt{\delta})R
    +\dfrac{2bd}{\sqrt{\delta}}.
\end{align*}
If $U(x)\leq L(1+|x|^{\nu})$, one may additionally use
$U_r\leq L(1+r^{\nu})$.

\noindent Collecting the above bounds yields
\begin{align*}
    \mathcal{A}\Psi
    \leq-\eta\Psi+C_{\mathcal{K}}\mathbb{I}_{\mathcal{K}},
    \qquad \eta,C_{\mathcal{K}}>0,
\end{align*}
where
\begin{align*}
    C_{\mathcal{K}}
    :=(\eta+\gamma J)\exp\left(
    a\beta U_r+(a\beta+b\sqrt{\delta})R
    +\dfrac{2bd}{\sqrt{\delta}}\right).
\end{align*}
Thus, for any fixed $\delta,\beta,\gamma>0$ and $a\in(0,1)$,
a sufficiently small choice of $b>0$ gives the required
Lyapunov drift condition \eqref{LypCond}. The constants and
the compact set may depend on $\delta$.
\end{proof}
\section{Minorization} 
\subsection{Invariant measure identification}
\noindent Consider the single-SDE representation of the system in \eqref{mal_SDE}.Let $f\in C_c^{\infty}(\mathbb{R}^{n\times m}\times \mathbb{R}^{n\times m})$ be a test function. Applying Ito's formula yields the standard infinitesimal generator:
\begin{align}
    \mathcal{A}f=\langle F,\nabla_z f\rangle+\text{Tr}(D\nabla_z^2f)=\langle T_{\delta},\nabla_x f\rangle-\langle\nabla_xU+\gamma p-\dfrac{\gamma}{\beta}\Gamma_{\delta},\nabla_pf\rangle+\dfrac{\gamma}{\beta}\sum_{j=1}^m\sum_{i=1}^n\sum_{k=1}^n[M_{\delta}]_{ik}\dfrac{\partial^2f}{\partial p_{ij}\partial p_{kj}}
\end{align}
The corresponding adjoint operator is given by (here div is understood as the complete double-sum, i.e it is a number):
\begin{align}
    \mathcal{A}^*\pi=-\text{div}_x\left((T_{\delta}(p)\pi\right)+\text{div}_{p}(\nabla U(x)\pi)+\text{div}_p\left(\left(\gamma p-\dfrac{\gamma}{\beta}\Gamma_{\delta}(p)\right)\pi\right)+\dfrac{\gamma}{\beta}\sum_{j=1}^m\sum_{i=1}^n\sum^n_{k=1}\dfrac{\partial^2([M_{\delta}(p)]_{ik}\pi)}{\partial p_{ij}\partial p_{kj}}\label{adjoint_op}
\end{align}
Now we simplify the diffusion term of Eq.\eqref{adjoint_op}:
\begin{align}
    \sum_{k=1}^n\dfrac{\partial([M_{\delta}(p)]_{ik}\pi)}{\partial p_{kj}}=\sum_{k=1}^n\left[\dfrac{\partial [M_{\delta}(p)]_{ik}}{\partial p_{kj}}\pi+[M_{\delta}(p)]_{ik}\dfrac{\partial \pi}{\partial p_{kj}}\right]=\pi [\Gamma_{\delta}(p)]_{ij}+\sum_{k=1}^n[M_{\delta}(p)]_{ik}\dfrac{\partial \pi}{\partial p_{kj}},\nonumber
\end{align}
where the last inequality follows from the identity in Eq.\eqref{gammaDiv}. Let $\pi=C^{-1}e^{-\beta H(x,p)}$, so that $\partial_{p_{ij}}\pi=-\beta\pi[T_{\delta}(p)]_{ij}$ and $\partial_{x_{ij}}\pi=-\beta\pi(\partial U(x)/\partial x_{ij})$. Substituting to the former equation yields: 
\begin{align}
    \sum_{k=1}^n\dfrac{\partial([M_{\delta}(p)]_{ik}\pi)}{\partial p_{kj}}&=\pi [\Gamma_{\delta}(p)]_{ij}-\beta\pi\sum_{k=1}^n[M_{\delta}(p)]_{ik}[T_{\delta}(p)]_{kj}\nonumber\\
    &=\pi [\Gamma_{\delta}(p)]_{ij}-\beta\pi[M_{\delta}(p)T_{\delta}(p)]_{ij}=\pi [\Gamma_{\delta}(p)]_{ij}-\beta\pi[p]_{ij},\nonumber
\end{align}
Since by definition $M_{\delta}(p)T_{\delta}(p)=(I_n+\delta pp^{\top})^{1/2}(I_n+\lambda pp^{\top})^{-1/2}p=p$. Differentiating the former equation with respect to $p_{ij}$ and summing over $j,i$ yields:
\begin{align}
    \sum_{j=1}^m\sum_{i=1}^n\sum^n_{k=1}\dfrac{\partial^2([M_{\delta}(p)]_{ik}\pi)}{\partial p_{ij}\partial p_{kj}}&=\sum_{j=1}^m\sum_{i=1}^n\dfrac{\partial(\pi [\Gamma_{\delta}(p)]_{ij}-\beta\pi p_{ij})}{\partial p_{ij}}\nonumber\\
    &=\sum_{j=1}^m\sum_{i=1}^n\pi\dfrac{ \partial[\Gamma_{\delta}(p)]_{ij}}{\partial p_{ij}}+[\Gamma_{\delta}(p)]_{ij}\dfrac{\partial \pi}{\partial p_{ij}}-\beta\pi-\beta p_{ij}\dfrac{\partial\pi}{\partial p_{ij}}\nonumber\\
    &= \pi\text{div}_p\Gamma_{\delta}(p)-\beta\pi\langle \Gamma_{\delta}(p),T_{\delta}(p)\rangle-\beta\pi mn+\beta^2\langle p,T_{\delta}(p)\rangle.\label{diffusion_term}
\end{align}
Expanding the divergence in the friction term of Eq.\eqref{adjoint_op} , yields:
\begin{align}
    \text{div}_p\left(\left(\gamma p-\dfrac{\gamma}{\beta}\Gamma_{\delta}(p)\right)\pi\right)&=\sum_{j=1}^m\sum_{i=1}^n \dfrac{\partial\pi}{\partial p_{ij}}\left(\gamma p_{ij}-\dfrac{\gamma}{\beta}[\Gamma_{\delta}(p)]_{ij}\right)+\pi \dfrac{\partial(\gamma p_{ij}-\gamma\beta^{-1}[\Gamma_{\delta}(p)]_{ij})}{\partial p_{ij}}\nonumber\\
    &=\sum_{j=1}^m\sum_{i=1}^n -\gamma\beta\pi[T_{\delta}(p)]_{ij}p_{ij}+\dfrac{\gamma}{\beta}\pi[T_{\delta}(p)]_{ij}[\Gamma_{\delta}(p)]_{ij}+\gamma\pi-\dfrac{\gamma}{\beta}\pi\dfrac{ \partial[\Gamma_{\delta}(p)]_{ij}}{\partial p_{ij}}\nonumber\\
    &=-\gamma\beta\pi\langle p,T_{\delta}(p)\rangle+\gamma\pi\langle \Gamma_{\delta}(p),T_{\delta}(p)\rangle+\gamma\pi mn-\dfrac{\gamma}{\beta}\pi\text{div}\Gamma_{\delta}(p).\label{friction_term}
\end{align}
Therefore, in view of Eq.\eqref{diffusion_term}-\eqref{friction_term}, the friction term and the diffusion terms of Eq.\eqref{adjoint_op} cancel out. Moreover, the Hamiltonian terms of Eq.\eqref{adjoint_op} cancel each other as well. Indeed we have:
\begin{align}
   -\text{div}_x\left((T_{\delta}(p)\pi\right)+\text{div}_{p}(\nabla U(x)\pi)&=-\sum_{i=1}^n\sum_{j=1}^m\dfrac{\partial [T_{\delta}(p)]_{ij}\pi}{\partial x_{ij}}+\sum_{i=1}^n\sum_{j=1}^m\dfrac{\partial(\partial U(x)/\partial x_{ij})\pi}{\partial p_{ij}}\nonumber\\
   &=\beta\sum_{i=1}^n\sum_{j=1}^m[T_{\delta}(p)]_{ij}(\partial U(x)/\partial x_{ij})-\beta\sum_{i=1}^n\sum_{j=1}^m(\partial U(x)/\partial x_{ij})[T_{\delta}(p)]_{ij}=0\label{hamiltonian_term}
\end{align}
Plugging Eq.\eqref{diffusion_term}-\eqref{friction_term}-\eqref{hamiltonian_term} into Eq.\eqref{adjoint_op}, finally given: $\mathcal{A}^*\pi=0.$
\subsection{Hormander's Condition}
Consider the Stratonovich form of the Hamiltonian system of Eq.\eqref{mal_singleSDE}:
\begin{align}
    dZ_t=\left(F(Z_t)-\text{Tr}\left(\nabla D^{1/2}(Z_t)[D^{1/2}(Z_t)]\right)\right)dt+\sqrt{2D(Z_t)}\odot d\hat{W}_t.
\end{align}
For reasons of clarity, we re-write the former Stratonovich Equation coordinate-wise, for all $i,j\in\{1,\ldots,2n\}$,
\begin{align}
    d[Z_t]_{ij}=\left([F(Z_t)]_{ij}-\sum_{r=1}^{2n}\sum_{k=1}^{2n}\dfrac{\partial [D^{1/2}(Z_t)]_{ik}}{\partial z_{rj}}[D^{1/2}(Z_t)]_{rk}\right)dt+\sqrt{2}\sum_{k=1}^{2n} [D^{1/2}(Z_t)]_{ik}\odot d[\hat{W}_t]_{kj}.
\end{align}
Since, $z=(x;p)$ and $D=\beta^{-1}[0,0;0,\gamma M_{\delta}(p)]$, we get $[D^{1/2}(z)]_{rk}=0$ for all $r,k\leq n$. Moreover, $D$ depends only on $p$, hence all the partial derivatives with respect to $\partial z_{rj}$ are $0$ when $r\leq n$, for any $j\in\{1,\ldots,2n\}$. Hence one obtains, for all $i,j\in\{1,\ldots,n\}$:
\begin{align}
    d[X_t]_{ij}&=[T_{\delta}(P_t)]_{ij}dt\nonumber\\
    d[P_t]_{ij}&=\left(-\dfrac{\partial\nabla U(X_t)}{\partial x_{ij}}-\gamma [P_t]_{ij}+\dfrac{\gamma}{\beta}[\Gamma_{\delta}(P_t)]_{ij}\right.\nonumber\\
    &\left.-\dfrac{\gamma}{\beta}\sum_{r=1}^{n}\sum_{k=1}^{n}\dfrac{\partial [M_{\delta}^{1/2}(P_t)]_{ik}}{\partial p_{rj}}[M_{\delta}^{1/2}(P_t)]_{rk}\right)dt+\sqrt{\dfrac{2\gamma}{\beta}}\sum_{k=1}^n[M_{\delta}^{1/2}(P_t)]_{ik}\odot d[{W}_t]_{kj}\nonumber
\end{align}
Equivalently, we write:
\begin{align}
    dX_t&=T_{\delta}(P_t)dt,\nonumber\\
    dP_t&=\left(-\nabla U(X_t)-\gamma P_t+\dfrac{\gamma}{\beta}\Gamma_{\delta}(P_t)-\text{Tr}(\nabla M_{\delta}^{1/2}(P_t)[M_{\delta}^{1/2}(P_t)]\right)dt+\sqrt{\dfrac{2\gamma}{\beta}M_{\delta}(P_t)}\odot dW_t.\label{stratonovich}
\end{align}
Or more simply:
\begin{align}
    Z_t=\begin{pmatrix} dX_t \\ dP_t \end{pmatrix}=\begin{pmatrix}T_{\delta}(P_t)\\
    G_{\delta}(X_t,P_t)\end{pmatrix}dt+\begin{pmatrix}0 \\ \Sigma_{\delta}(P_t)\end{pmatrix}\odot d\hat{W}_t,\label{stratonovich2}
\end{align}
where $G_{\delta},\Sigma_{\delta}$ denote the corresponding quantities of Eq. \ref{stratonovich}. Now let us define the indicator matrix $[\mathcal{I}^{kj}]_{ab}=\mathcal{I}\{a=k,b=j\}\in\mathbb{R}^{n\times m}$. The corresponding vector fields are defined as follows:
\begin{align}
    Y_0f=\begin{pmatrix} T_{\delta}(P_t)\\ G_{\delta}(X_t,P_t) \end{pmatrix}^{\top}\begin{pmatrix} \nabla_xf\\ \nabla_pf   \end{pmatrix},\ Y_{kj}f=\begin{pmatrix}
        0\\\Sigma_{\delta}(P_t)\mathbb{I}^{kj}\end{pmatrix}^{\top}\begin{pmatrix} \nabla_xf\\ \nabla_pf   
    \end{pmatrix}
\end{align}
so that the infinitesimal generator of Eq.\ref{stratonovich2} can be written as:
\begin{align}
    \mathcal{A}f=Y_0f+\dfrac{1}{2}\sum_{k=1}^n\sum_{j=1}^nY^2_{kj}f,
\end{align}
and the define the Lie-bracket operation as $[X,Y]f=XYf-YXf$. The fields $Y_{kj}$ act only in the momentum directions, so $\mathcal A$ is not elliptic on the full state space. However, if the Lie-algebra $\mathcal{L}$ generated by $\{Y_0,Y_{11},\ldots,Y_{1n},Y_{21},\ldots,Y_{nn}\}$ has maximal rank, i.e. spans $\mathbb{R}^{2n\times m}$, then the generator is known to be hypo-elliptic. In particular we show that the Hormander condition holds. Then the Lie-bracket can be expressed as:
\begin{align}
    &[Y_0,Y_{kj}]f=Y_0Y_{kj}f-Y_{kj}Y_0f\nonumber\\&=\begin{pmatrix} T_{\delta}\\ G_{\delta} \end{pmatrix}^{\top}\begin{pmatrix} \nabla_x\langle \nabla_pf,\Sigma_{\delta}\mathbb{I}^{kj}\rangle\\ \nabla_p\langle\nabla_pf,\Sigma_{\delta}\mathbb{I}^{kj}\rangle   \end{pmatrix}-\begin{pmatrix}
        0\\\Sigma_{\delta}\mathbb{I}^{kj}\end{pmatrix}^{\top}\begin{pmatrix} \nabla_x \langle \nabla_x f, T_{\delta}\rangle+\langle \nabla_pf,G_{\delta}\rangle\\ \nabla_p \langle\nabla_x f ,T_{\delta}\rangle+\langle\nabla_pf,G_{\delta}\rangle  
    \end{pmatrix}\nonumber\\
    &=T_{\delta}\nabla^2_{xp}f[\Sigma_{\delta}\mathbb{I}^{kj}]+T_{\delta}\nabla_x(\Sigma_{\delta}\mathbb{I}^{kj})^{\top}[\nabla_p f]+G_{\delta}\nabla^2_{pp}f[\Sigma_{\delta}\mathbb{I}^{kj}]+G_{\delta}\nabla_p(\Sigma_{\delta}\mathbb{I}^{kj})^{\top}[\nabla_p f]\nonumber\\
    &-\Sigma_{\delta}\mathbb{I}^{kj}\nabla^2_{xp}f[T_{\delta}]-\Sigma_{\delta}\mathbb{I}^{kj}\nabla_pT_{\delta}^{\top}[\nabla_xf]-\Sigma_{\delta}\mathbb{I}^{kj}\nabla^2_{pp}f[G_{\delta}]-\Sigma_{\delta}\mathbb{I}^{kj}\nabla_p G_{\delta}^{\top}[\nabla_pf]\nonumber\\
    &=G_{\delta}\nabla_p(\Sigma_{\delta}\mathbb{I}^{kj})^{\top}[\nabla_p f]-\Sigma_{\delta}\mathbb{I}^{kj}\nabla_pT_{\delta}^{\top}[\nabla_xf]-\Sigma_{\delta}\mathbb{I}^{kj}\nabla_p G_{\delta}^{\top}[\nabla_pf]
\end{align}
Equilavently:
\begin{align}
    [Y_0,Y_{kj}]f=\begin{pmatrix}
        -\Sigma_{\delta}\mathbb{I}^{kj}\nabla_pT_{\delta}^{\top}\\ G_{\delta}\nabla_p(\Sigma_{\delta}\mathbb{I}^{kj})^{\top}-\Sigma_{\delta}\mathbb{I}^{kj}\nabla_p G_{\delta}^{\top}
    \end{pmatrix}^{\top}\begin{pmatrix}
        \nabla_x f\\\nabla_pf
    \end{pmatrix}
\end{align}
Suppose now that there exist weights $a_{kj},b_{kj}\in\mathbb{R}$, such that $\forall(x,p)\in\mathbb{R}^{2n\times m}$:
\begin{align}
    \sum_{k=1}^{n}\sum_{j=1}^{m}a_{kj}Y_{kj}+b_{kj}[Y_0,Y_{kj}]=(0,0)^{\top}
\end{align}
Since the positional element of $Y_{kj}$ is $0$, we obtain the condition 
\begin{align}
&\sum_{k=1}^{n}\sum_{j=1}^{m}b_{kj}([Y_0,Y_{kj}])_x=0\Leftrightarrow \sum_{k=1}^{n}\sum_{j=1}^{m}b_{kj}\nabla_pT_{\delta}[\Sigma_{\delta}\mathbb{I}^{kj}]=0\nonumber\\
&\Leftrightarrow \nabla_pT_{\delta}\left[\Sigma_{\delta}\sum_{k=1}^{n}\sum_{j=1}^{m}b_{kj}\mathbb{I}^{kj}\right]=0\Leftrightarrow  \sum_{k=1}^{n}\sum_{j=1}^{m}b_{kj}\mathbb{I}^{kj}=0\Leftrightarrow B=0,
\end{align}
where the second to last step follows from the fact that $\nabla_pT_{\delta}[A]=0\Leftrightarrow A=0$ (see Lemma \ref{boundLaplacian}) and the invertibility of $\Sigma_{\delta}$. This implies in turn, that:
\begin{align}
    \sum_{k=1}^{n}\sum_{j=1}^{m}a_{kj}(Y_{kj})_p=0\Leftrightarrow \Sigma_{\delta}\sum_{k=1}^{n}\sum_{j=1}^{m}a_{kj}\mathbb{I}^{kj}=0\Leftrightarrow A=0.
\end{align}
Therefore $a_{kj},b_{kj}=0$ for all $k\in\{1,\ldots,n\}$ and $j\in\{1,\ldots,m\}$, and the $2nm$ vectors $\{Y_{kj}\}_{k,j=1}^{n,m},\{[Y_0,Y_{kj}]\}_{k,j=1}^{n,m}$ are linearly indepedent for all $(x,p)\in \mathbb{R}^{2n\times m}$. Assume that $U\in C^\infty$ and that the solution is nonexplosive, as ensured by the Lyapunov condition \eqref{LypCond}. Since the coefficient of the SDE in Eq.\eqref{stratonovich2} are time-homogenous and $C^{\infty}$-smooth, Theorem 1(i) and Theorem 3 in \cite{Ichihara1974}, yields that the transition Kernel admits a density $p_t(z,w)\in C^{\infty}(\mathbb{R}_+,\mathbb{R}^{2n\times m},\mathbb{R}^{2n\times m})$ for all $t\geq0$, and therefore the process is strong-Feller, i.e.:
\begin{align}
    R_t(z,dw)=p_t(z,w)dw.
\end{align}
\subsection{Controllability}
First we aim to show a reachability property for the transition kernel $R_t$ of Eq.\eqref{mal_hamiltonian_SDE}. Since the diffusion coefficient is degenerate, in the sense that noise acts only in the momentum, we base our proof on the Support Theorem originally provided \cite{StroockVaradhan1972}. This essentially reduces the argument to a control problem. Throughout this section, $|\cdot|$ denotes the Frobenius norm and $|\cdot|_2=\sigma_1(\cdot)$ the spectral norm.\\

\noindent Let $T>0$. We define the space of controls and the path space respectively by
\begin{align*}
    \mathcal{U}_T:=C^{\infty}([0,T];\mathbb{R}^{n\times m}),\ \ \mathcal{W}_T:=C([0,T];\mathbb{R}^{2n\times m}),
\end{align*}
where $\mathcal{W}_T$ is equipped with the uniform norm $|z|_{\infty}=\sup_{t\leq T}|z(t)|$. One considers the Stratonovich form of the kinetic system in Eq.\eqref{stratonovich2} and substitutes the Brownian increment $dW_t$ by $\nu(t)dt$, where $\nu\in\mathcal{U}_T$. Therefore, the corresponding control system is:
\begin{align}
    \dot{x}(t)=T_{\delta}(p(t)),\ \ \dot{p}(t)=G_{\delta}(x(t),p(t))+\Sigma_{\delta}(p(t))\nu(t),\ \ (x(0),p(0))=z_0.\label{controlled_SDE}
\end{align}
Henceforth, we denote by $z^{\nu}=(x^{\nu},p^{\nu})\in\mathcal{W}_T$ the solution of the control system \eqref{controlled_SDE}, whenever it exists on $[0,T]$. The set of controlled trajectories and the accessibility set with initial condition $z_0$ are defined as:
\begin{align*}
    \mathcal{C}_T(z_0)&=\{z^{\nu}:\nu\in\mathcal{U}_T\}\subseteq\mathcal{W}_T,\\
    \mathcal{A}_T(z_0)&=\{z^{\nu}(T):\nu\in\mathcal{U}_T\}\subseteq\mathbb{R}^{2n\times m}.
\end{align*}
Moreover, by Lemma \ref{sec_SDE_drift} the process started at $z_0$ is non-explosive, hence the law of the process $(Z_t)_{t\in[0,T]}$, namely $\mu_{z_0}=\mathbb{P}_{z_0}\circ Z^{-1}$, is a Borel probability measure on $\mathcal{W}_T$. Its support
\begin{align*}
    S_T(z_0)=\text{supp}\mu_{z_0}
\end{align*}
is defined as the smallest closed subset of $\mathcal{W}_T$ such that $\mathbb{P}_{z_0}\left((Z_t)_{t\in[0,T]}\in S_T(z_0)\right)=1$. Similarly, $\text{supp}R_T(z_0,\cdot)$ denotes the support of $R_T(z_0,\cdot)=\mathbb{P}_{z_0}(Z_T\in\cdot)$ on $\mathbb{R}^{2n\times m}$. Both supports are well-defined by Lemma \ref{support1} below.\\

\noindent For SDEs with bounded and sufficiently smooth coefficients, the Support Theorem identifies $S_T(z_0)$ with the closure of $\mathcal{C}_T(z_0)$. Since the drift of Eq.\eqref{stratonovich2} grows superlinearly, we only establish the inclusion
\begin{align*}
    \overline{\mathcal{C}_T(z_0)}\subseteq S_T(z_0)
\end{align*}
by a truncation argument (Lemma \ref{support2}), which is the direction needed for positivity. The reverse inclusion, on the level of $R_T(z_0,\cdot)$, follows directly from the bounded velocity of the relativistic kinetic energy (Lemma \ref{support3}).

Since we use relativistic velocity, one can reason that the accessibility set cannot exceed the lightcone on the position-momentum diagram. The speed limit is dictated by the spectral norm. Indeed, recall that $T_{\delta}(p)=U\tau(\Sigma)V^{\top}$, where $p=U\Sigma V^{\top}$ and $\tau(s)=s(1+\delta s^2)^{-1/2}<\delta^{-1/2}$ (see Remark \ref{SpectralProperties}.5). Therefore every singular value of $T_{\delta}(p)$ is strictly smaller than $\delta^{-1/2}$, i.e.
\begin{align}
    |T_{\delta}(p)|_2<\delta^{-1/2},\ \ \forall p\in\mathbb{R}^{n\times m}.\label{speed_limit}
\end{align}
Note that this bound is sharper than the Frobenius bound $|T_{\delta}(p)|\leq\sqrt{d/\delta}$ of Remark \ref{SpectralProperties}.5, which is attained only when all $d$ singular values saturate simultaneously. Moreover, since $\tau:[0,\infty)\to[0,\delta^{-1/2})$ is a bijection, $T_{\delta}$ is a smooth bijection from $\mathbb{R}^{n\times m}$ onto the open spectral ball
\begin{align*}
    \mathcal{Q}_{\delta}:=\{q\in\mathbb{R}^{n\times m}:|q|_2<\delta^{-1/2}\},
\end{align*}
with smooth inverse given by
\begin{align}
    T_{\delta}^{-1}(q)=q(I-\delta q^{\top}q)^{-1/2},\ \ q\in\mathcal{Q}_{\delta}.\label{taming_inverse}
\end{align}
Hence steering the system reduces to prescribing a smooth velocity profile inside $\mathcal{Q}_{\delta}$.
\begin{lemma}\label{velocity_profile}
    Let $T>0$, $x_0,x_1\in\mathbb{R}^{n\times m}$ with $|x_1-x_0|_2<\delta^{-1/2}T$, and $q_0,q_1\in\mathcal{Q}_{\delta}$. Then there exists $\psi\in C^{\infty}([0,T];\mathcal{Q}_{\delta})$ such that
    \begin{align*}
        \psi(0)=q_0,\ \ \psi(T)=q_1,\ \ \int_0^{\top}\psi(t)dt=x_1-x_0.
    \end{align*}
\end{lemma}
\begin{proof}
    Let $\rho=\sqrt{\delta}|x_1-x_0|_2/T\in[0,1)$ and define the smooth step function:
    \begin{align}
        f(s)=\begin{cases}
            0\ &,\text{for } s\le 0\\
            \dfrac{e^{-1/s}}{e^{-1/s}+e^{-1/(1-s)}} &,\text{for } 0< s< 1\\
            1\ &,\text{for } s\geq 1
        \end{cases}\label{smooth_step}
    \end{align}
    It is well known that $f\in C^{\infty}(\mathbb{R};[0,1])$ and that all its derivatives vanish at $s=0$ and $s=1$. Moreover, $f(s)+f(1-s)=1$, hence $\int_0^1f(s)ds=1/2$. Now let us set the parameters:
    \begin{align*}
        \ell=\dfrac{(1-\rho)T}{4},\ \text{and } \psi^{*}=\dfrac{(x_1-x_0)-\ell(q_0+q_1)/2}{T-\ell}.
    \end{align*}
    Then we define $\psi:[0,T]\to\mathbb{R}^{n\times m}$ as follows:
    \begin{align}
        \psi(t)=\begin{cases}
            q_0+f(t/\ell)(\psi^{*}-q_0)&,0\leq t\leq\ell\\
            \psi^{*}&,\ell\leq t\leq T-\ell\\
            \psi^{*}+f((t-T+\ell)/\ell)(q_1-\psi^{*})&,T-\ell\leq t\leq T.
        \end{cases}\label{velocity_def}
    \end{align}
    Since all derivatives of $f$ vanish at $0$ and $1$, the three pieces glue smoothly, i.e. $\psi\in C^{\infty}([0,T];\mathbb{R}^{n\times m})$ with $\psi(0)=q_0$ and $\psi(T)=q_1$. Moreover, in view of $\int_0^1f(s)ds=1/2$, one obtains
    \begin{align*}
        \int_0^{\top}\psi(t)dt&=\ell\dfrac{q_0+\psi^{*}}{2}+(T-2\ell)\psi^{*}+\ell\dfrac{\psi^{*}+q_1}{2}\\
        &=\ell\dfrac{q_0+q_1}{2}+(T-\ell)\psi^{*}=x_1-x_0.
    \end{align*}
    It remains to show that $\psi$ takes values in $\mathcal{Q}_{\delta}$. Since $|q_0|_2,|q_1|_2<\delta^{-1/2}$, $|x_1-x_0|_2=\rho\delta^{-1/2}T$ and $\ell>0$, we have
    \begin{align*}
        |\psi^{*}|_2\leq\dfrac{|x_1-x_0|_2+\ell(|q_0|_2+|q_1|_2)/2}{T-\ell}<\dfrac{\rho\delta^{-1/2}T+\ell\delta^{-1/2}}{T-\ell}=\delta^{-1/2}\dfrac{1+3\rho}{3+\rho}\leq\delta^{-1/2},
    \end{align*}
    where the last inequality follows from $\rho<1$. Hence $\psi^{*}\in\mathcal{Q}_{\delta}$. Finally, since $f$ takes values in $[0,1]$, each $\psi(t)$ is a convex combination of $q_0,\psi^{*},q_1\in\mathcal{Q}_{\delta}$. Since $\mathcal{Q}_{\delta}$ is convex, it follows that $\psi(t)\in\mathcal{Q}_{\delta}$ for all $t\in[0,T]$.
\end{proof}
\begin{lemma}\label{controllability}
    For every $z_0=(x_0,p_0)\in\mathbb{R}^{2n\times m}$ and $T>0$, the accessibility set is the open lightcone
    \begin{align}
        \mathcal{A}_T(z_0)=\{(x,p)\in\mathbb{R}^{2n\times m}:|x-x_0|_2<\delta^{-1/2}T\}.\label{lightcone}
    \end{align}
\end{lemma}
\begin{proof}
    First, let $z_1=(x_1,p_1)\in\mathcal{A}_T(z_0)$, i.e. $z_1=z^{\nu}(T)$ for some $\nu\in\mathcal{U}_T$. In view of Eq.\eqref{controlled_SDE} and Eq.\eqref{speed_limit}, we have
    \begin{align*}
        |x_1-x_0|_2=\left|\int_0^{\top}T_{\delta}(p^{\nu}(t))dt\right|_2\leq\int_0^{\top}|T_{\delta}(p^{\nu}(t))|_2dt<\delta^{-1/2}T,
    \end{align*}
    where the strict inequality holds since the integrand is continuous and strictly smaller than $\delta^{-1/2}$.\\

    \noindent Conversely, let $z_1=(x_1,p_1)$ with $|x_1-x_0|_2<\delta^{-1/2}T$, and set $q_0=T_{\delta}(p_0)$, $q_1=T_{\delta}(p_1)$, so that $q_0,q_1\in\mathcal{Q}_{\delta}$. Lemma \ref{velocity_profile} provides $\psi\in C^{\infty}([0,T];\mathcal{Q}_{\delta})$ with $\psi(0)=q_0$, $\psi(T)=q_1$ and $\int_0^{\top}\psi(t)dt=x_1-x_0$. Since $T_{\delta}$ is invertible on $\mathcal{Q}_{\delta}$ by Eq.\eqref{taming_inverse}, and $\Sigma_{\delta}(p)$ is invertible due to $M_{\delta}(p)\succeq I$, we define:
    \begin{align*}
        x(t)&=x_0+\int_0^{\top}\psi(s)ds,\\
        p(t)&=T_{\delta}^{-1}(\psi(t)),\\
        \nu(t)&=\Sigma^{-1}_{\delta}(p(t))\left(\dot{p}(t)-G_{\delta}(x(t),p(t))\right).
    \end{align*}
    All three maps are smooth, hence $\nu\in\mathcal{U}_T$. By construction, $(x,p)$ solves the control system \eqref{controlled_SDE} with $(x(0),p(0))=(x_0,T_{\delta}^{-1}(q_0))=z_0$. Since the coefficients of Eq.\eqref{controlled_SDE} are smooth, its solution is unique, hence $z^{\nu}=(x,p)$ and
    \begin{align*}
        z^{\nu}(T)=\left(x_1,T_{\delta}^{-1}(q_1)\right)=(x_1,p_1)=z_1.
    \end{align*}
    That is $z_1\in\mathcal{A}_T(z_0)$, which concludes the proof.
\end{proof}
\subsection{Positive Recurrence}
Let $(E,\mathrm{d})$ be a separable metric space and denote by $\mathcal{B}_{\epsilon}(x)=\{x'\in E:\mathrm{d}(x,x')<\epsilon\}$ the open ball of radius $\epsilon$ around $x\in E$. In particular, for $E=\mathcal{W}_T$ the ball $\mathcal{B}_{\epsilon}(z)$ is the $\epsilon$-envelope around the path $z$, whereas for $E=\mathbb{R}^{2n\times m}$ it is the usual Frobenius ball. Now we give a useful characterization of the support.
\begin{lemma}\label{support1}
    Let $\mu$ be a Borel probability measure on a separable metric space $E$. Then the set
    \begin{align*}
        F=\{e\in E:\mu(\mathcal{B}_{\epsilon}(e))>0,\ \forall\epsilon>0\}
    \end{align*}
    is the smallest closed subset of $E$ of full $\mu$-measure, i.e. $F=\text{supp}\mu$.
\end{lemma}
\begin{proof}
    Let $x\in F^c$, then there exists $\epsilon>0$ such that $\mu(\mathcal{B}_{\epsilon}(x))=0$. Now choose arbitrary $x'\in\mathcal{B}_{\epsilon}(x)$ and let $0<\epsilon'\leq\epsilon-\mathrm{d}(x,x')$. For any $x''\in\mathcal{B}_{\epsilon'}(x')$ it follows that $\mathrm{d}(x'',x)\leq\mathrm{d}(x'',x')+\mathrm{d}(x',x)<\epsilon$, i.e. $\mathcal{B}_{\epsilon'}(x')\subseteq\mathcal{B}_{\epsilon}(x)$. This in turn implies that $\mu(\mathcal{B}_{\epsilon'}(x'))=0$, hence $x'\in F^c$. Therefore $\mathcal{B}_{\epsilon}(x)\subseteq F^c$, so $F^c$ is an open set covered by $\mu$-null balls. Moreover, since $E$ is separable, by Lindelof's Lemma, any open cover can be reduced to a countable subcover. Therefore, there exists a sequence $\{\mathcal{B}_{\epsilon_k}(x_k)\}_{k=1}^{\infty}$ of $\mu$-null balls that cover $F^c$. Then,
    \begin{align*}
        F^c\subseteq\bigcup_{k=1}^{\infty}\mathcal{B}_{\epsilon_k}(x_k)\Rightarrow\mu(F^c)\leq\sum_{k=1}^{\infty}\mu(\mathcal{B}_{\epsilon_k}(x_k))=0.
    \end{align*}
    Hence $F$ is a closed set of full measure. It remains to show that $F\subseteq A$ for any closed set $A$ with $\mu(A)=1$. Suppose there exists $x\in F\setminus A$. Since $A^c$ is open, there exists $\epsilon>0$ such that $\mathcal{B}_{\epsilon}(x)\subseteq A^c$. Since $A^c$ is $\mu$-null, it follows that $\mu(\mathcal{B}_{\epsilon}(x))=0$, which contradicts $x\in F$. Therefore $F\subseteq A$, which concludes the proof.
\end{proof}
\noindent In \cite{IkedaWatanabe1981} the Support Theorem for Stratonovich SDEs, is provided with the conditions that the drift is bounded and Lipschitz and the diffusion coefficient is $C^2$ with bounded derivatives up to order 2. Since the drift of SDE\eqref{stratonovich2} grows superlinearly, we first apply their result to a truncated system and then remove the truncation by a localization argument.
\begin{lemma}\label{support2}
    Let $z_0\in\mathbb{R}^{2n\times m}$, $T>0$ and $z^{\nu}\in\mathcal{C}_T(z_0)$. Then, for every $\epsilon>0$, we have
    \begin{align}
        \mathbb{P}_{z_0}\left(\sup_{t\in[0,T]}|Z_t-z^{\nu}(t)|<\epsilon\right)>0.\label{tube_positive}
    \end{align}
    Equivalently, $z^{\nu}\in S_T(z_0)$, i.e. $\overline{\mathcal{C}_T(z_0)}\subseteq S_T(z_0)$.
\end{lemma}
\begin{proof}
    Let us denote the drift and diffusion coefficients of SDE\eqref{stratonovich2} by $b$ and $\sigma$ respectively, so that $dZ_t=b(Z_t)dt+\sigma(Z_t)\odot d\hat{W}_t$. Fix $\epsilon>0$ and $R>|z^{\nu}|_{\infty}+\epsilon$. Then we define the truncated coefficients by $b_R=\chi_Rb$ and $\sigma_R=\chi_R\sigma$, where $\chi_R\in C_c^{\infty}$ is such that $\chi_R=1$ on $\mathcal{B}_{R+1}(0)$. Since $b,\sigma\in C^{\infty}$, we have $b_R,\sigma_R\in C_b^{\infty}$. Let $Z^R$ be the unique strong solution of
    \begin{align*}
        dZ^R_t=b_R(Z^R_t)dt+\sigma_R(Z^R_t)\odot d\hat{W}_t,\ \ Z^R_0=z_0,
    \end{align*}
    driven by the same Brownian motion as $Z$. Notice that $|z^{\nu}(t)|\leq|z^{\nu}|_{\infty}<R$ for all $t\in[0,T]$. Since both systems have the same coefficients on $\mathcal{B}_{R+1}(0)$, the path $z^{\nu}$ is also a controlled trajectory of the truncated system, i.e. $z^{\nu}\in\mathcal{C}^R_T(z_0)$. Then Theorem 8.1 in [cite Watanabe] yields $S^R_T(z_0)=\overline{\mathcal{C}^R_T(z_0)}$ for the support of the law of $Z^R$, hence $z^{\nu}\in S^R_T(z_0)$. In view of Lemma \ref{support1} with $E=\mathcal{W}_T$, we obtain
    \begin{align*}
        \mathbb{P}_{z_0}(\Omega_{R,\epsilon})>0,\ \ \text{where } \Omega_{R,\epsilon}=\left\{\sup_{t\in[0,T]}|Z^R_t-z^{\nu}(t)|<\epsilon\right\}.
    \end{align*}
    Next, define the exit time $\tau_R=\inf\{t\geq0:|Z^R_t|\geq R\}$. On $[0,\tau_R]$ the process $Z^R$ takes values in $\overline{\mathcal{B}_R(0)}\subset\mathcal{B}_{R+1}(0)$, where both systems have the same coefficients. By pathwise uniqueness for SDE\eqref{stratonovich2}, we obtain $Z_{t\land\tau_R}=Z^R_{t\land\tau_R}$ for all $t\geq0$, a.s. Moreover, on $\Omega_{R,\epsilon}$ it holds that:
    \begin{align*}
        |Z^R_t|\leq|z^{\nu}(t)|+|Z^R_t-z^{\nu}(t)|<|z^{\nu}|_{\infty}+\epsilon<R,\ \ \forall t\in[0,T].
    \end{align*}
    Consequently $\tau_R>T$, and therefore $Z_t=Z^R_t$ for all $t\in[0,T]$, a.s. on $\Omega_{R,\epsilon}$. Finally,
    \begin{align*}
        \mathbb{P}_{z_0}\left(\sup_{t\in[0,T]}|Z_t-z^{\nu}(t)|<\epsilon\right)\geq\mathbb{P}_{z_0}(\Omega_{R,\epsilon})>0.
    \end{align*}
    Since $\epsilon>0$ was arbitrary, Lemma \ref{support1} yields $z^{\nu}\in S_T(z_0)$. Since $S_T(z_0)$ is closed, we conclude that $\overline{\mathcal{C}_T(z_0)}\subseteq S_T(z_0)$.
\end{proof}
\begin{lemma}\label{support3}
    For every $z_0=(x_0,p_0)\in\mathbb{R}^{2n\times m}$ and $T>0$,
    \begin{align}
        \text{supp}R_T(z_0,\cdot)=\overline{\mathcal{A}_T(z_0)}=\{(x,p)\in\mathbb{R}^{2n\times m}:|x-x_0|_2\leq\delta^{-1/2}T\}.\label{support_lightcone}
    \end{align}
\end{lemma}
\begin{proof}
    The second identity follows from Lemma \ref{controllability}, since the closure of the open lightcone is the closed one. First we show that $\overline{\mathcal{A}_T(z_0)}\subseteq\text{supp}R_T(z_0,\cdot)$. Let $z_1\in\mathcal{A}_T(z_0)$, i.e. $z_1=z^{\nu}(T)$ for some $\nu\in\mathcal{U}_T$, and fix $\epsilon>0$. Since $|Z_T-z_1|\leq\sup_{t\in[0,T]}|Z_t-z^{\nu}(t)|$, Lemma \ref{support2} yields
    \begin{align*}
        R_T(z_0,\mathcal{B}_{\epsilon}(z_1))=\mathbb{P}_{z_0}(|Z_T-z_1|<\epsilon)\geq\mathbb{P}_{z_0}\left(\sup_{t\in[0,T]}|Z_t-z^{\nu}(t)|<\epsilon\right)>0.
    \end{align*}
    Since $\epsilon>0$ was arbitrary, Lemma \ref{support1} with $E=\mathbb{R}^{2n\times m}$ gives $z_1\in\text{supp}R_T(z_0,\cdot)$. Hence $\mathcal{A}_T(z_0)\subseteq\text{supp}R_T(z_0,\cdot)$, and since the support is closed, the inclusion extends to $\overline{\mathcal{A}_T(z_0)}$.\\

    \noindent Conversely, writing $Z_t=(X_t,P_t)$, Eq.\eqref{mal_hamiltonian_SDE} and Eq.\eqref{speed_limit} yield almost surely
    \begin{align*}
        |X_T-x_0|_2=\left|\int_0^{\top}T_{\delta}(P_t)dt\right|_2\leq\int_0^{\top}|T_{\delta}(P_t)|_2dt\leq\delta^{-1/2}T.
    \end{align*}
    Therefore $\overline{\mathcal{A}_T(z_0)}$ is a closed set of full $R_T(z_0,\cdot)$-measure, and Lemma \ref{support1} implies $\text{supp}R_T(z_0,\cdot)\subseteq\overline{\mathcal{A}_T(z_0)}$.
\end{proof}
\begin{corollary}\label{accessibility}
    Let $C\subseteq\mathbb{R}^{2n\times m}$ be a bounded set and let $z_1=(x_1,p_1)\in\mathbb{R}^{2n\times m}$ be the target. Define $T_C=\sqrt{\delta}\sup_{(x,p)\in C}|x_1-x|_2$. Then for any $\epsilon>0$ and $t>0$ with $t\geq T_C$, we have
    \begin{align}
        R_t(z,\mathcal{B}_{\epsilon}(z_1))>0,\ \ \forall z\in C.\label{accessibility_ineq}
    \end{align}
\end{corollary}
\begin{proof}
    Let $z=(x,p)\in C$. By the definition of $T_C$, we have $|x_1-x|_2\leq\delta^{-1/2}T_C\leq\delta^{-1/2}t$. Hence $z_1\in\text{supp}R_t(z,\cdot)$ by Lemma \ref{support3}, and Eq.\eqref{accessibility_ineq} follows from Lemma \ref{support1}.
\end{proof}
\subsection{Minorization}
We adapt the local-density argument of Assumption 2.1 and Lemma 2.3 in \cite{MattinglyStuartHigham2002}. Henceforth, let $\text{Leb}$ denote the Lebesgue measure on $\mathbb{R}^{2n\times m}$ and $\mathscr{B}(\mathbb{R}^{2n\times m})$ the corresponding Borel $\sigma$-algebra.
\begin{proposition}\label{minorization}
    For every nonempty compact set $C\subseteq\mathbb{R}^{2n\times m}$, there exist $\bar{T}_C>0$, $\alpha_C>0$ and a probability measure $\nu_C$ on $\mathscr{B}(\mathbb{R}^{2n\times m})$, such that
    \begin{align}
        R_{\bar{T}_C}(z,A)\geq\alpha_C\nu_C(A),\ \ \forall z\in C,\ A\in\mathscr{B}(\mathbb{R}^{2n\times m}).\label{minorization_ineq}
    \end{align}
    If $\text{int}C\neq\emptyset$, then $\nu_C$ may be chosen such that $\nu_C(C)=1$.
\end{proposition}
\begin{proof}
    First, suppose that $\text{int}C\neq\emptyset$. Choose $z_*=(x_*,p_*)\in\text{int}C$, fix $s>0$ and let $\epsilon>0$ be such that $\overline{\mathcal{B}_{\epsilon}(z_*)}\subset\text{int}C$. In view of Lemma \ref{support3}, we have $z_*\in\text{supp}R_s(z_*,\cdot)=\{(x,p):|x-x_*|_2\leq\delta^{-1/2}s\}$, i.e. $z_*$ is reachable from itself at time $s$. Therefore, by Lemma \ref{support1} and Eq.\eqref{kernelDensity}, one obtains
    \begin{align*}
        0<R_s(z_*,\mathcal{B}_{\epsilon}(z_*))=\int_{\mathcal{B}_{\epsilon}(z_*)}p_s(z_*,w)dw.
    \end{align*}
    Hence there exists $w_*\in\mathcal{B}_{\epsilon}(z_*)$ such that $p_s(z_*,w_*)>0$. Since the density $p_s$ is jointly continuous, there exist $\epsilon_0,\epsilon_1>0$ with $\overline{\mathcal{B}_{\epsilon_0}(z_*)},\overline{\mathcal{B}_{\epsilon_1}(w_*)}\subset\text{int}C$, and a constant $c_0>0$, such that
    \begin{align}
        p_s(v,w)\geq c_0,\ \ \forall v\in\mathcal{B}_{\epsilon_0}(z_*),\ w\in\mathcal{B}_{\epsilon_1}(w_*).\label{density_lower}
    \end{align}
    Next, we fix a single entrance time $\tau$ for the whole set $C$, namely
    \begin{align}
        \tau\geq T_C=\sqrt{\delta}\sup_{(x,p)\in C}|x_*-x|_2.\label{entrance_time}
    \end{align}
    Corollary \ref{accessibility} with target $z_1=z_*$ and $\epsilon=\epsilon_0$ yields $R_{\tau}(z,\mathcal{B}_{\epsilon_0}(z_*))>0$ for all $z\in C$. Moreover, $p_{\tau}$ is jointly continuous and hence bounded on the compact set $C\times\overline{\mathcal{B}_{\epsilon_0}(z_*)}$. Thus, by dominated convergence, the map $z\mapsto R_{\tau}(z,\mathcal{B}_{\epsilon_0}(z_*))=\int_{\mathcal{B}_{\epsilon_0}(z_*)}p_{\tau}(z,w)dw$ is continuous on $C$. Since $C$ is compact, the minimum is attained, i.e.
    \begin{align*}
        \kappa_C:=\min_{z\in C}R_{\tau}(z,\mathcal{B}_{\epsilon_0}(z_*))>0.
    \end{align*}
    Now let $A\in\mathscr{B}(\mathbb{R}^{2n\times m})$. The Chapman-Kolmogorov equation and Eq.\eqref{density_lower} yield, for all $z\in C$:
    \begin{align*}
        R_{\tau+s}(z,A)&=\int_{\mathbb{R}^{2n\times m}}R_s(v,A)R_{\tau}(z,dv)\geq\int_{\mathcal{B}_{\epsilon_0}(z_*)}R_s(v,A)R_{\tau}(z,dv)\\
        &\geq\int_{\mathcal{B}_{\epsilon_0}(z_*)}\left(\int_{A\cap\mathcal{B}_{\epsilon_1}(w_*)}p_s(v,w)dw\right)R_{\tau}(z,dv)\\
        &\geq c_0\text{Leb}(A\cap\mathcal{B}_{\epsilon_1}(w_*))R_{\tau}(z,\mathcal{B}_{\epsilon_0}(z_*))\geq c_0\kappa_C\text{Leb}(A\cap\mathcal{B}_{\epsilon_1}(w_*)).
    \end{align*}
    Therefore, setting
    \begin{align*}
        \bar{T}_C=\tau+s,\ \ \nu_C(A)=\dfrac{\text{Leb}(A\cap\mathcal{B}_{\epsilon_1}(w_*))}{\text{Leb}(\mathcal{B}_{\epsilon_1}(w_*))},\ \ \alpha_C=c_0\kappa_C\text{Leb}(\mathcal{B}_{\epsilon_1}(w_*))>0,
    \end{align*}
    one recovers Eq.\eqref{minorization_ineq}. Since $\mathcal{B}_{\epsilon_1}(w_*)\subset C$, we additionally get $\nu_C(C)=1$. For an arbitrary compact set $C$, there exists $r>0$ such that $C\subseteq\overline{\mathcal{B}_r(0)}$. The closed ball has nonempty interior, hence its minorization restricts to $C$, with $\bar{T}_C,\alpha_C,\nu_C$ being those of $\overline{\mathcal{B}_r(0)}$.
\end{proof}
\section{Stability \& Lyapunov Condition}
Let us define the auxiliary functions $y=y(x,p)=p+\gamma^{-1}\nabla U(x)$ and $F=F(x,p)=U(x)+wK_{\delta}(y(x,p))$, for some $w>0$. Since $U(x)$ is coersive (see dissipativity Remark) and semi-convex (see Assumption), it attains a global minimum. Henceforth, without the loss of generality, we further assume that $U(x)\geq0$. Moreover, by construction, $K_{\delta}(p)\geq0$. Now we define the Lyapunov function:
\begin{align}
    \mathcal{V}(x,p)=\exp[\eta F(x,p)]=\exp\left[\eta U(x)+\eta wK_{\delta}(y)\right],\ \ \  w,\eta>0. \label{mal_lyapunov}
\end{align}
Henceforth, let us define $\Xi_{n+1}:=\sqrt{2\lambda\gamma\beta^{-1}M_{\delta}(P_n)}\xi_{n+1}$, for reasons of brevity. Multiplying the definition of $y$ by the factor $(1-\gamma\lambda)$ and plugging in Eq.\eqref{mal_hamiltonian_MC}, we obtain:
\begin{align*}
    (1-\gamma\lambda)Y_n&=(1-\gamma\lambda)P_n+\gamma^{-1}\nabla U(X_n)-\lambda\nabla U(X_n)\\
    (1-\gamma\lambda)Y_n&=P_{n+1}+\lambda\nabla U(X_n)-\lambda\gamma\beta^{-1}\text{div}[M_{\delta}(P_n)]-\Xi_{n+1}+\gamma^{-1}\nabla U(X_n)-\lambda\nabla U(X_n).
\end{align*}
Rearranging and adding zero, yields:
\begin{align*}
    Y_{n+1}&=(1-\gamma\lambda)Y_n+\gamma^{-1}(\nabla U(X_{n+1})-\nabla U(X_n))+\lambda\gamma\beta^{-1}\text{div}[M_{\delta}(P_n)]+\Xi_{n+1}\\
    Y_{n+1}&=(1-\gamma\lambda)Y_n+\Delta_n+\Xi_{n+1}, \ \text{ with } \Delta_n:=\gamma^{-1}(\nabla U(X_{n+1})-\nabla U(X_n))+\lambda\gamma\beta^{-1}\text{div}[M_{\delta}(P_n)].
\end{align*}
Notice that $\Delta_n\in\mathcal{F}_n$, since $X_{n+1}=X_n+\lambda T_{\delta}(P_n)$. Now we aim to decompose the difference $F_{n+1}-F_n$, into a $\mathcal{F}_n$-measurable term and a $\mathcal{F}_n$-conditionally Gaussian one:
\begin{align}
    F(X_{n+1},P_{n+1})-F(X_n,P_n)&=F(X_{n+1},P_{n+1}-\Xi_{n+1})-F(X_n,P_n)\nonumber\\
    &+F(X_{n+1},P_{n+1})-F(X_{n+1},P_{n+1}-\Xi_{n+1}):=\mathcal{T}_1+\mathcal{T}_2.\label{eq_disLyap_01}
\end{align}
Additionally, we have that:
\begin{align}
    F(X_{n+1},P_{n+1}-\Xi_{n+1})&=U(X_{n+1})+wK_{\delta}(Y(X_{n+1},P_{n+1}-\Xi_{n+1}))\nonumber\\
    &=U(X_{n+1})+wK_{\delta}(P_{n+1}-\Xi_{n+1}+\gamma^{-1}\nabla U(X_{n+1}))\nonumber\\
    &=U(X_{n+1})+wK_{\delta}(Y_{n+1}-\Xi_{n+1})\nonumber\\
    &=U(X_{n+1})+wK_{\delta}((1-\gamma\lambda)Y_n+\Delta_n).    
\end{align}
Therefore one writes:
\begin{align}
    \mathcal{T}_1&=U(X_{n+1})-U(X_n)+w[K_{\delta}((1-\gamma\lambda)Y_n+\Delta_n)-K_{\delta}(Y_n)]\nonumber\\
    \mathcal{T}_2&=w[K_{\delta}((1-\gamma\lambda)Y_n+\Delta_n+\Xi_{n+1})-K_{\delta}((1-\gamma\lambda)Y_n+\Delta_n)]\label{LyapunovTerms}
\end{align}
\subsection{Supporting Lemmas}
\begin{lemma}\label{FenchelYoung}
    Young-type inequality for the Kinetic energy. For any $\epsilon\in(0,1)$, it holds:
    \begin{align}
        \langle A,T_{\delta}(p)\rangle\leq \dfrac{1}{\epsilon^2}K_{\delta}(A)+\epsilon^2K_{\delta}(p).
    \end{align}
\end{lemma}
\begin{proof}
    Postponed.
\end{proof}
\begin{lemma}\label{GaussianConce}
    The Tsirelson-Ibragimov-Sudakov Inequality. Let $X$ be a normal standard r.v. and $f$ a $L$-Lipschitz function, then
    \begin{align}
        \log\mathbb{E}\left[e^{\lambda[f(X)-\mathbb{E}f(X)]}\right]\leq\dfrac{\lambda^2L^2}{2},\ \forall\lambda\in\mathbb{R}.
    \end{align}
\end{lemma}
\begin{proof}
    In Chapter 5.4 Theorem 5.5 of \citet{BoucheronLugosiMassart2013}.
\end{proof}
\begin{lemma}
\begin{align*}
    K_{\delta}((1-\gamma\lambda)Y_n+\Delta_n)-K_{\delta}(Y_n)\leq-\gamma\lambda K_{\delta}(Y_n)+\hat{C}_2(K_{\delta}(Y_n)+K_{\delta}(P_n))+\hat{C}_1,
\end{align*}
where $\hat{C}_2=\dfrac{\lambda dL_1}{\delta^{1/2}}e^{\lambda L_1 L_{\delta}}$ and $\hat{C}_1=\dfrac{\lambda\gamma d}{\beta}(d+1/2)+\dfrac{2\lambda d^2L_1}{\delta^{3/2}}e^{\lambda L_1 L_{\delta}}+\dfrac{\lambda dL_0}{\gamma \delta}e^{\lambda L_1 L_{\delta}}$.
\end{lemma}
\begin{proof}
Since $K_{\delta}(p)$ is convex by Lemma \ref{HessianBound}, we have $\forall t\in[0,1]$ and $p_1,p_2\in\mathbb{R}^{m\times n}$, that $K_{\delta}(tp_1+(1-t)p_2)\leq tK_{\delta}(p_1)+(1-t)K_{\delta}(p_2)$. Choosing $p_2=0$, yields $K_{\delta}(tp)\leq tK_{\delta}(p)$. Moreover by Remark \ref{SpectralProperties}.5, $K_{\delta}(p)$ is $L_{\delta}$-Lipschitz, with constant  $L_{\delta}=\sqrt{d/\delta}$. Therefore we write:
\begin{align}
    K_{\delta}((1-\gamma\lambda)Y_n+\Delta_n)&\leq K_{\delta}((1-\gamma\lambda)Y_n)+K_{\delta}((1-\gamma\lambda)Y_n+\Delta_n)-K_{\delta}((1-\gamma\lambda)Y_n)\nonumber\\
    &\leq (1-\gamma\lambda)K_{\delta}(Y_n)+L_{\delta}|\Delta_n|.\nonumber
\end{align}
Equivalently, this leads to: \begin{equation}
    K_{\delta}((1-\gamma\lambda)Y_n+\Delta_n)-K_{\delta}(Y_n)\leq-\gamma\lambda K_{\delta}(Y_n)+L_{\delta}|\Delta_n|\label{inq_kinetic1}.
\end{equation}
Moreover by the definition of $y$ and Remark \ref{SpectralProperties}.4, we obtain:
\begin{equation}
    |\nabla U(X_n)|\leq\gamma|Y_n-P_n|\leq \gamma|Y_n|+\gamma|P_n|\leq \gamma\delta^{1/2}K_{\delta}(Y_n)+\gamma\delta^{1/2}K_{\delta}(P_n)+2d\gamma\delta^{-1/2}.\label{inq_kinetic2}
\end{equation}
Recall the $L_0/L_1$-condition of the potential, i.e. $|\nabla^2U(x)|\leq L_0+L_1|\nabla U(x)|,\ \forall x\in\mathbb{R}^{m\times n}$. We define the auxiliary function $\phi(t)=L_0+L_1|\nabla U(x_t)|$ with $x_t=x+ty$, so that we can apply Grownwall's Lemma later. Additionally we define the upper-right Duni derivative:
\begin{equation}
    D^+\phi(t)=\limsup_{\epsilon\to 0}\dfrac{\phi(t+\epsilon)-\phi(t)}{\epsilon}
\end{equation}
Since $U$ is semi-convex, the Alexandrov Hessian exists a.e. and due to $\nabla U$ being locally Lipschitz, the fundamental theorem of calculus yields:
\begin{align}
    |\nabla U(x_{t+\epsilon})-\nabla U(x_t)|& \leq\int_t^{t+\epsilon}|\nabla^2U(x_s)[y]|ds\leq\int_t^{t+\epsilon}|\nabla^2U(x_s)|| y|ds\leq\int_t^{t+\epsilon}|y|(L_0+L_1|\nabla U(x_s)|)ds.\label{inq_kinetic3}
\end{align}
Now the reverse triangular inequality gives:
\begin{align}
    \phi(t+\epsilon)-\phi(t)&\leq L_1 |\nabla U(x_{t+\epsilon})-\nabla U(x_t)|\leq  |y|L_1 \int_t^{t+\epsilon}\phi(s)ds,\nonumber
\end{align}
Dividing by $\epsilon$, and taking the $\limsup$ as $\epsilon\to 0$ yields:
\begin{align}
    D^+\phi(t)\leq L_1|y|\phi(t)&\Rightarrow \phi(t)\leq \phi(0)\exp[{L_1t|y|}]\nonumber\\
    \Rightarrow(L_0+L_1|\nabla U(x_t)|)&\leq\exp[L_1t|y|](L_0+L_1|\nabla U(x)|),\label{gronwall1}
\end{align}
where the second to last step follows from Gronwall's Lemma. In view of Eq.\eqref{inq_kinetic3}, setting $\epsilon=1$, $t=0$, $x=X_n$, and $y=\lambda T_{\delta}(P_n)$, we get:
\begin{align}
    |\nabla U(X_{n+1})-\nabla U(X_n)|=|\nabla U(X_n+\lambda T_{\delta}(P_n))-\nabla U(X_n)|\leq e^{\lambda L_1 L_{\delta}}\lambda L_{\delta}(L_0+L_1|\nabla U(X_n)|).
\end{align}
Finally, Lemma [?] provides the bound $\text{div}[M_{\delta}(P_n)]\leq(d+1/2)\sqrt{d\delta}$. Combining the former result with the inequalities \eqref{inq_kinetic1},\eqref{inq_kinetic2} and \eqref{inq_kinetic3} yields:
\begin{align}
    K_{\delta}((1-\gamma\lambda)Y_n+\Delta_n)-K_{\delta}(Y_n)\leq-\gamma\lambda K_{\delta}(Y_n)+\hat{C}_2(K_{\delta}(Y_n)+K_{\delta}(P_n))+\hat{C}_1.\label{inq_kinetic4}
\end{align}
\end{proof}
\begin{remark}
    Instead of the elementary bound $K_{\delta}(p)\geq {\delta}^{-1/2}|p|-d\delta^{-1}$ at \eqref{inq_kinetic2}, I could use $K_{\delta}(p)\geq \dfrac{p^2}{2d+\sqrt{\delta}|p|}$, which gives $|p|\leq(\sqrt{2d}+\sqrt{\delta})K_{\delta}(p)+\sqrt{2}/4$ instead of $|p|\leq \sqrt{\delta}K_{\delta}(p)+d^{1/2}\delta^{-1/2}$.
\end{remark}
\begin{lemma}
\begin{align*}
    F(X_{n+1},P_{n+1}-\Xi_{n+1})-F(X_n,P_n):=\mathcal{T}_1\leq\lambda A_yK_{\delta}(Y_n)+\lambda A_p K_{\delta}(P_n)+\lambda\hat{C}_3,
\end{align*}
where $\hat{C}_3=\dfrac{\lambda L_0de^{\lambda L_1 L_{\delta}}}{2\delta}+\dfrac{\lambda\gamma L_1d^2e^{\lambda L_1 L_{\delta}}}{\delta^{3/2}}+w\left[\dfrac{\gamma d(d+1/2)}{\beta}+\dfrac{L_0de^{\lambda L_1 L_{\delta}}}{\gamma\delta}+\dfrac{2L_1d^2e^{\lambda L_1 L_{\delta}}}{\delta^{3/2}}\right]$,\\ and $A_y=\gamma(\epsilon^{-2}-w)+\dfrac{\lambda de^{\lambda L_1 L_{\delta}}}{2\delta^{1/2}}+w\dfrac{dL_1e^{\lambda L_1 L_{\delta}}}{\delta^{1/2}}$, $A_p=\gamma(\epsilon^2-1)+\dfrac{\lambda de^{\lambda L_1 L_{\delta}}}{2\delta^{1/2}}+w\dfrac{dL_1e^{\lambda L_1 L_{\delta}}}{\delta^{1/2}}$.
\end{lemma}
\begin{proof}
We proceed to control the difference $U(X_{n+1})-U(X_n)$. Applying the Taylor expantion yields:
\begin{align}
    U(X_n+\lambda T_{\delta}(P_n))-U(X_n)&=\langle \nabla U(X_n),\lambda T_{\delta}(P_n)\rangle+\int_0^1(1-\tau)\nabla^2 U(X_n+t\lambda T_{\delta}(P_n))[\lambda T_{\delta}(P_n),\lambda T_{\delta}(P_n)]d\tau\nonumber\\
    &\leq \lambda \langle \nabla U(X_n),T_{\delta}(P_n)\rangle+e^{\lambda L_1 L_{\delta}}(L_0+L_1|\nabla U(X_n)|)\lambda^2d/(2\delta),\label{inq_kinetic5}
\end{align}
where ones obtains the inequality on the integral term by using the $L_0/L_1$-condition and the established bound of Eq.\eqref{inq_kinetic2} with $t=1$, $x=X_n$ and $y=\lambda T_{\delta}(P_n)$. Moreover, by the definition of $Y_n$, we write:
\begin{align}
    \langle \nabla U(X_n),T_{\delta}(P_n)\rangle&=\gamma\langle Y_n,T_{\delta}(P_n)\rangle-\gamma \langle P_n,T_{\delta}(P_n)\rangle\nonumber\\
    &\leq \gamma\epsilon^{-2}K_{\delta}(Y_n)+\gamma(\epsilon^2-1)K_{\delta}(P_n),\label{inq_kinetic6}\end{align}
    where the last inequality follows from Lemma \ref{FenchelYoung}.
In view of Eq.\eqref{inq_kinetic4} and plugging \eqref{inq_kinetic6} and \eqref{inq_kinetic2} into \eqref{inq_kinetic5}, we obtain the final bound:
\begin{align}
    \mathcal{T}_1&=U(X_{n+1})-U(X_n)+w[K_{\delta}((1-\gamma\lambda)Y_n+\Delta_n)-K_{\delta}(Y_n)]\nonumber\\
    &\leq \gamma\lambda\epsilon^{-2}K_{\delta}(Y_n)+\gamma\lambda(\epsilon^2-1)K_{\delta}(P_n)\nonumber\\
    &+\dfrac{\lambda^2 de^{\lambda L_1 L_{\delta}}}{2\delta}(L_0+L_1(\gamma\delta^{1/2}K_{\delta}(Y_n)+\gamma\delta^{1/2}K_{\delta}(P_n)+2d\gamma\delta^{-1/2}))\nonumber\\
    &+w\left[-\gamma\lambda K_{\delta}(Y_n)+\hat{C}_2(K_{\delta}(Y_n)+K_{\delta}(P_n))+\hat{C}_1\right].\nonumber
\end{align}
Equivalently, after rearranging the above result:
\begin{align}
    F(X_{n+1},P_{n+1}-\Xi_{n+1})-F(X_n,P_n):=\mathcal{T}_1\leq\lambda A_yK_{\delta}(Y_n)+\lambda A_p K_{\delta}(P_n)+\lambda\hat{C}_3.\label{T1_bound}
\end{align}
\end{proof}
\begin{lemma}
Define the auxiliary quantity $B_n=(1-\gamma\lambda)Y_n+\Delta_n$, which is $\mathcal{F}_n$-measurable by construction. Moreover, recall that $\Xi_{n+1}=\sqrt{2\gamma\lambda\beta^{-1}M_{\delta}(P_n)}\xi_{n+1}$, hence $\Xi_{n+1}$ conditioned at $\mathcal{F}_n$ is a centered Gaussian with covariance $\Sigma=2\gamma\lambda\beta^{-1}M_{\delta}(P_n)\oplus I_d$.
\begin{align*}
    \log\mathbb{E}^{\mathcal{F}_n}\left[e^{\eta w[K_{\delta}(B_n+\Xi_{n+1})-K_{\delta}(B_n)]}\right]&\leq \eta\lambda A_3K_{\delta}(P_n)+\eta\lambda \hat{C}_4.
\end{align*}
Additionally, $A_3=\dfrac{w\gamma d\delta}{\beta}+\dfrac{\eta w^2\gamma d}{\beta}$, whereas $\hat{C}_4=\dfrac{w\gamma dm}{\beta}+\dfrac{\eta w^2\gamma d(1+d)}{\delta\beta}$.
\end{lemma}
\begin{proof}
We start by providing two elementary bounds:
\begin{align}
    \text{Tr}(M_{\delta})&=\sum_{i=1}^m\sqrt{1+\delta \sigma_i^2(p)}=m+\delta\sum_{i=1}
^m\left(\dfrac{1}{\delta}\sqrt{1+\delta\sigma^2_i(p)}-1\right)=m+\delta K_{\delta}(p),\label{M_bound1}\\
|M_{\delta}(p)|_2&=\sqrt{1+\delta\sigma_1^2(p)}\leq 1+\delta^{1/2}\sigma_1(p)\leq1+\delta^{1/2}|p|\leq1+d+\delta K_{\delta}(p)\label{M_bound2}.
\end{align}
Since $T_{\delta}$ is $1-$Lipschitz (see Remark \ref{SpectralProperties}.2), we obtain the following by f.t.c.:
\begin{align}
    K_{\delta}(x+h)-K(x)&=\int_0^1\langle T_{\delta}(x+th),h\rangle dt\nonumber\\
    K_{\delta}(x+h)-K(x)-\langle T_{\delta}(x),h\rangle&=\int_0^1\langle T_{\delta}(x+th)-T(x),h\rangle dt\nonumber\\
    &\leq\int_0^1|h||T_{\delta}(x+th)-T_{\delta}(x)|dt\nonumber
\end{align}
Therefore,
\begin{align}
     K_{\delta}(x+h)-K(x)-\langle T_{\delta}(x),h\rangle\leq \int_0^1 t|DT_{\delta}|_2|h|^2dt\leq \dfrac{1}{2}|h|^2.\label{inq_kinetic7}
\end{align}
Next, we apply Eq.\eqref{inq_kinetic7} with $x=B_n$ and $h=\Xi_{n+1}$, to yield:
\begin{align}
    \mathbb{E}\left[K_{\delta}(B_n+\Xi_{n+1})-K_{\delta}(B_n)|\mathcal{F}_n\right]&=\mathbb{E}\left[\langle T_{\delta}(B_n),\Xi_{n+1}\rangle|\mathcal{F}_n\right]\nonumber\\&+\mathbb{E}\left[K_{\delta}(B_n+\Xi_{n+1})-K_{\delta}(B_n)-\langle T_{\delta}(B_n),\Xi_{n+1}\rangle |\mathcal{F}_n\right]\nonumber\\
    &\leq (1/2)\mathbb{E}\left[|\Xi_{n+1}|^2|\mathcal{F}_n\right]\leq (1/2)\text{Tr}(\Sigma)\nonumber\\
    &\leq \gamma\lambda\beta^{-1}(m+\delta K_{\delta}(P_n))
\end{align}
Finally, let us define the auxiliary function $\psi(\zeta)=K_{\delta}(x+A\zeta)-K_{\delta}(x)$. We apply Lemma \ref{GaussianConce} for the conditional expectation $\mathbb{E}[\cdot|\mathcal{F}_n]$ with $\psi(\zeta)=K_{\delta}(x+A\zeta)-K_{\delta}(x)$ for $\zeta=\xi_{n+1}$, $A=\Sigma^{1/2}$ and $x=B_n$. First, we show that $\psi$ is indeed Lipschitz:
\begin{align}
    |\psi(\zeta_1)-\psi(\zeta_2)|&\leq |K_{\delta}(x+A\zeta_1)-K_{\delta}(x+A\zeta_2)|\leq |T_{\delta}|_2|A(\zeta_1-\zeta_2)|.
\end{align}
Moreover, due to Eq.\eqref{M_bound2}, we have that $|\Sigma^{1/2}_n|_2\leq\sqrt{2\gamma\lambda\beta^{-1}}|M(P_n)|_2^{1/2}\leq \left(2\gamma\lambda\beta^{-1}[1+d+\delta K_{\delta}(P_n)]\right)^{1/2}$. Therefore, applying Lemma \ref{GaussianConce} for $\tau=\eta w$, we obtain:
\begin{align}
    \log\mathbb{E}^{\mathcal{F}_n}\left[e^{\eta w[K_{\delta}(B_n+\Xi_{n+1})-K_{\delta}(B_n)]}\right]&\leq \eta w\mathbb{E}^{\mathcal{F}_n}\left[K_{\delta}(B_n+\Xi_{n+1})-K_{\delta}(B_n)\right]+\dfrac{\eta^2w^2d}{2\delta}\dfrac{2\gamma\lambda}{\beta}(1+d+\delta K_{\delta}(P_n))\nonumber\\
    &\leq \dfrac{\eta w}{2}\dfrac{2\gamma\lambda}{\beta}d(m+\delta K_{\delta}(P_n))+\dfrac{\eta^2w^2d\gamma\lambda}{\delta\beta}(1+d+\delta K_{\delta}(P_n))\nonumber\\
    &\leq \eta\lambda A_3K_{\delta}(P_n)+\eta\lambda \hat{C}_4.\label{T2_bound}
\end{align}
where the second to last inequality follows from the elementary bound of Eq.\eqref{M_bound1}.
\end{proof}
\subsection{The Lyapunov bound}
Let us write $\mathcal{V}_n=\mathcal{V}(X_n,P_n)=\exp[\eta F(X_n,P_n)]$, conditioning on $\mathcal{F}_n$ yields:
\begin{align}
    \mathbb{E}[\mathcal{V}_{n+1}|\mathcal{F}_n]=\mathcal{V}_n\mathbb{E}[e^{\eta[F(X_{n+1},P_{n+1})-F(X_n,P_n)]}|\mathcal{F}_n]=\mathcal{V}_n\mathbb{E}[e^{\eta[\mathcal{T}_1+\mathcal{T}_2]}|\mathcal{F}_n].\nonumber
\end{align}
In view of Eq.\eqref{LyapunovTerms}, the factor $e^{\eta\mathcal{T}_1}$ is $\mathcal{F}_n$-measurable. After factorizing the former quantity out of the expectation, one bounds it by Eq.\eqref{T1_bound}. Then one applies Eq.\eqref{T2_bound} for the leftover expectation, to conclude with:
\begin{align}
    \mathbb{E}[\mathcal{V}_{n+1}|\mathcal{F}_n]&\leq\mathcal{V}_n\exp\left[\eta\lambda A_yK_{\delta}(Y_n)+\eta\lambda A_p K_{\delta}(P_n)+\eta\lambda\hat{C}_3\right]\exp\left[\eta\lambda A_3K_{\delta}(P_n)+\eta\lambda \hat{C}_4\right]\nonumber\\
    &\leq\mathcal{V}_n\exp\left[\eta\lambda A_yK_{\delta}(Y_n)+\eta\lambda \left(A_p+A_3\right) K_{\delta}(P_n)+\eta\lambda(\hat{C}_3+\hat{C}_4)\right]\nonumber\\
    &\leq \mathcal{V}_n\exp\left[\eta\lambda \hat{A}(K_{\delta}(Y_n)+K_{\delta}(P_n))+\eta\lambda\hat{C}_5\right],
\end{align}
with $\hat{A}=\max\{A_y,A_p+A_3\}$ and $\hat{C}_5=\hat{C}_3+\hat{C}_4$.\\
Let us define the compact set $\Omega=\{x,p\in\mathbb{R}^{n\times m}:K_{\delta}(p)+K_{\delta}(p+\gamma^{-1}\nabla U(x))\geq R_0\}$, where $R_0=-(\hat{C}_5+1)\hat{A}^{-1}.$
Then one immediately obtains:
\begin{align}
    \mathbb{E}[\mathcal{V}_{n+1}|\mathcal{F}_n]\leq e^{-\eta\lambda}\mathcal{V}_n, \ \forall (x,p)\in\Omega.
\end{align}
The former argument implies the following hyperparameters condition: $A_y\leq0$ and $A_p+A_3\leq0$.
\bibliographystyle{plainnat}
\section{References} \vspace{-2.0em} \renewcommand\refname{} 

\end{document}